\documentclass[acmsmall]{acmart}

\usepackage{graphicx}

\usepackage{xspace}
\usepackage[english]{babel}

\usepackage[utf8]{inputenc}
\usepackage[T1]{fontenc}
\usepackage{booktabs}
\usepackage{nicefrac}
\usepackage{xcolor}
\usepackage{enumitem}
\usepackage{lipsum}
\usepackage{soul}
\usepackage{threeparttable}
\usepackage{amsmath}
\usepackage{amsthm}

\usepackage{amssymb} 
\usepackage{bm}
\usepackage{color}
\usepackage[ruled,linesnumbered]{algorithm2e}
\usepackage{mathtools}
\usepackage{multirow}
\usepackage{pifont}
\usepackage{pstricks-add}
\usepackage{subcaption}

\usepackage{tikz}
\usetikzlibrary{arrows.meta, positioning, shapes, decorations.pathreplacing, calc, backgrounds, fit}

\SetCommentSty{mycommfont}

\theoremstyle{definition}

\theoremstyle{definition}
\newtheorem*{remark}{Remark}

\definecolor{vegaBlue}{HTML}{1f77b4}  \definecolor{vegaOrange}{HTML}{ff7f0e}  \definecolor{vegaGreen}{HTML}{2ca02c}  \definecolor{lightGray}{HTML}{e0e0e0} 

\newcommand{\xmark}{\ding{55}}  \newcommand{\cmark}{\ding{51}}  

\newenvironment{proofsketch}{\par{\scshape Proof Sketch.}~}{\hfill$\qed$\par}

\newcommand{\parh}[1]{\smallskip\noindent\textbf{#1}}

\newcommand{\tool}{\textsc{ALVGL}\xspace}
\newcommand{\F}{Fig.}

\newcommand{\T}{Table}
\renewcommand{\S}{Sec.}
\newcommand{\A}{Alg.}

\newcommand{\Lem}{Lemma}

\newcommand{\Cor}{Corollary}

\newcommand{\Prop}{Proposition}
\newcommand{\ignore}[1]{}

\author{Pingchuan Ma}
\email{pma@zjut.edu.cn}
\affiliation{\institution{Zhejiang University of Technology}
  \country{China}
}

\author{Qixin Zhang}
\affiliation{\institution{Nanyang Technological University}
  \country{Singapore}
}

\author{Shuai Wang}
\email{shuaiw@cse.ust.hk}
\affiliation{\institution{The Hong Kong University of Science and Technology}
  \country{Hong Kong SAR}
}

\author{Dacheng Tao}
\affiliation{\institution{Nanyang Technological University}
  \country{Singapore}
}

\setcopyright{none}
 \renewcommand\footnotetextcopyrightpermission[1]{} 



\begin{document}
\title{Efficient Differentiable Causal Discovery via Reliable Super-Structure
Learning}

\begin{abstract}

Recently, differentiable causal discovery has emerged as a promising approach to
improve the accuracy and efficiency of existing methods. However, when applied
to high-dimensional data or data with latent confounders, these methods, often
based on off-the-shelf continuous optimization algorithms, struggle with the
vast search space, the complexity of the objective function, and the nontrivial
nature of graph-theoretical constraints. As a result, there has been a surge of
interest in leveraging super-structures to guide the optimization process.
Nonetheless, learning an appropriate super-structure at the right level of
granularity, and doing so efficiently across various settings, presents
significant challenges.

In this paper, we propose \tool, a novel and general enhancement to the
differentiable causal discovery pipeline. \tool employs a sparse and low-rank
decomposition to learn the precision matrix of the data. We design an ADMM
procedure to optimize this decomposition, identifying components in the
precision matrix that are most relevant to the underlying causal structure.
These components are then combined to construct a super-structure that is
provably a superset of the true causal graph. This super-structure is used to
initialize a standard differentiable causal discovery method with a more
focused search space, thereby improving both optimization efficiency and
accuracy.

We demonstrate the versatility of \tool by instantiating it across a range of
structural causal models, including both Gaussian and non-Gaussian settings,
with and without unmeasured confounders. Extensive experiments on synthetic and
real-world datasets show that \tool not only achieves state-of-the-art accuracy
but also significantly improves optimization efficiency, making it a reliable
and effective solution for differentiable causal discovery.

\end{abstract}

\maketitle

\section{Introduction}

Causal discovery is a fundamental problem in machine learning and
statistics~\cite{spirtes2000causation, zhang2008completeness}, aiming to infer
the underlying causal structure from observational data. It has wide-ranging
applications across fields such as epidemiology, economics, and the social
sciences. The task becomes especially challenging in high-dimensional settings
or when latent confounders are present, prompting the development of a diverse
set of methods, including constraint-based~\cite{spirtes2000causation,
zhang2008completeness, colombo2012learning, rohekar2021iterative},
score-based~\cite{tsamardinos2006max, tsirlis2018scoring, chen2021integer,
claassen2022greedy}, and more recently, differentiable
approaches~\cite{zheng2018dags, bello2022dagma, zhang2022truncated, vowels2022d,
bhattacharya2021differentiable}.

Differentiable methods are particularly appealing because they reformulate
causal discovery as a continuous optimization problem, enabling the use of
gradient-based algorithms to learn causal structures directly from data. This is
especially valuable in high-dimensional scenarios, where traditional methods
often falter due to combinatorial explosion and computational bottlenecks.
Despite their promise, differentiable approaches face significant challenges.
Optimizing the entire graph structure, typically represented as a weighted
adjacency matrix, is inherently difficult: the search space grows rapidly with
dimensionality, and the presence of latent confounders further complicates the
objective landscape and constraints. As a result, existing methods often fail to
converge to high-quality solutions within a reasonable timeframe and may end up
with a cyclic graph, which is even not a valid causal structure.

\begin{table}[h]
    \centering
    \caption{List of super-structure-guided optimizations.}
    \label{tab:feature}
    \resizebox{0.65\linewidth}{!}{\begin{tabular}{lcccc}
        \toprule
         & Type & Setting & Latent? & Super-structure\\
        \midrule
        MMHC~\cite{tsamardinos2006max} & Score-based & Faithful & \xmark & Skeleton  \\
        CAM~\cite{buhlmann2014cam} & Score-based & Lin-Gauss & \xmark & Super Skeleton  \\
        GLasso~\cite{ng2021reliable} & Score-based & Lin-Gauss & \xmark & Moralized DAG  \\
SPOT~\cite{ma2024scalable} & Differentiable & Lin-Gauss & \cmark & Skeleton  \\
        SDCD~\cite{nazaret2024stable} & Differentiable & Faithful & \xmark & Preselected Edges  \\
        DCD w/ GLasso & Differentiable & Lin-Gauss & \xmark & Moralized DAG  \\
        \midrule
        DCD w/ \tool & Differentiable & Lin-Gauss & \cmark & Super Structure  \\
        \bottomrule
    \end{tabular}}
\end{table}

Notably, these challenges are not exclusive to differentiable methods;
score-based approaches, which rely on discrete optimization, face similar
scalability issues. A widely adopted strategy in the score-based paradigm is to
constrain the search space using a super-structure: a superset of the true
causal graph that restricts the set of candidate edges~\cite{tsamardinos2006max,
ng2021reliable}.

\begin{figure}[t]
\centering

\resizebox{0.7\linewidth}{!}
{\begin{tikzpicture}[
    node distance=0.9cm,
    var/.style={circle, draw, minimum size=0.65cm, font=\small},
    edge/.style={-{Stealth[length=1.7mm]}, thick},
    biedge/.style={{Stealth[length=1.7mm]}-{Stealth[length=1.7mm]}, thick, vegaOrange},
    pruned/.style={draw=gray!50, dashed, thin},
    kept/.style={draw=vegaBlue, thick},
    annotation/.style={font=\scriptsize, align=center}
]

\begin{scope}[local bounding box=full]
    \node[annotation, font=\small\bfseries] at (0.75,1.6) {From Fully-Connected};

    \node[var] (x1) at (0,0.7) {$X_1$};
    \node[var] (x2) at (1.5,0.7) {$X_2$};
    \node[var] (x3) at (0,-0.7) {$X_3$};
    \node[var] (x4) at (1.5,-0.7) {$X_4$};

    \foreach \a/\b in {x1/x2, x1/x3, x1/x4, x2/x3, x2/x4, x3/x4}
        \draw[pruned] (\a) -- (\b);

    \node[annotation, text=red!70!black] at (0.75,-1.45) {Search space: $\mathcal{O}(d^2)$};
\end{scope}

\draw[-{Stealth[length=2mm]}, thick, gray] (2.3,0) -- (3.1,0);
\node[annotation, above, gray] at (2.7,0.07) {\tool};

\begin{scope}[xshift=3.9cm, local bounding box=super]
    \node[annotation, font=\small\bfseries] at (0.75,1.6) {From Super-Structure};

    \node[var] (y1) at (0,0.7) {$X_1$};
    \node[var] (y2) at (1.5,0.7) {$X_2$};
    \node[var] (y3) at (0,-0.7) {$X_3$};
    \node[var] (y4) at (1.5,-0.7) {$X_4$};

    \draw[kept] (y1) -- (y2);
    \draw[kept] (y1) -- (y3);
    \draw[kept] (y2) -- (y4);
    \draw[kept] (y3) -- (y4);

    \draw[pruned, opacity=0.3] (y1) -- (y4);
    \draw[pruned, opacity=0.3] (y2) -- (y3);

    \node[text=red, font=\tiny] at ($(y1)!0.5!(y4)$) {\ding{55}};
    \node[text=red, font=\tiny] at ($(y2)!0.5!(y3)$) {\ding{55}};

    \node[annotation, text=vegaGreen] at (0.75,-1.45) {Search space: $\mathcal{O}(dk^2)$};
\end{scope}

\draw[-{Stealth[length=2mm]}, thick, gray] (6.2,0) -- (7.0,0);
\node[annotation, above, gray] at (6.6,0.07) {DCD};

\begin{scope}[xshift=7.7cm, local bounding box=mag]
    \node[annotation, font=\small\bfseries] at (0.75,1.6) {Learned MAG};

    \node[var] (z1) at (0,0.7) {$X_1$};
    \node[var] (z2) at (1.5,0.7) {$X_2$};
    \node[var] (z3) at (0,-0.7) {$X_3$};
    \node[var] (z4) at (1.5,-0.7) {$X_4$};

    \draw[edge, vegaBlue] (z1) -- (z2);
    \draw[edge, vegaBlue] (z1) -- (z3);
    \draw[edge, vegaBlue] (z2) -- (z4);

    \draw[biedge] (z3) to[bend left=30] (z4);

\node[annotation, text=vegaOrange] at (0.75,-1.45) {$\leftrightarrow$: latent confounder};

\end{scope}

\end{tikzpicture}}

\caption{Learning a MAG from a fully-connected graph vs. from a super-structure
learned by \tool. DCD can be instantiated using any differentiable causal
discovery method. Here, $d$ is the number of variables and $k$ is the maximum
in-degree of the true causal graph.}
\label{fig:superstructure-comparison}
\end{figure}
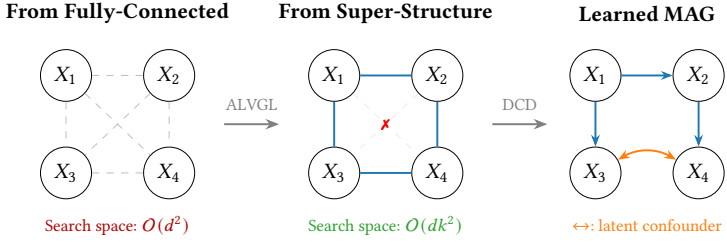 
This idea has recently been extended to differentiable methods, where a
super-structure is used to guide the optimization process~\cite{ma2024scalable,
nazaret2024stable}. However, skeleton-based approaches such as SPOT typically
incur high computational overhead due to the large number of conditional
independence tests required, which limits their scalability in high-dimensional
settings. In contrast, GLasso~\cite{ng2021reliable} provides a more efficient
alternative by learning a moralized DAG using graphical Lasso, offering better
scalability. Similarly, SDCD~\cite{nazaret2024stable} selects edges via a
constraint-free optimization process that, in the linear Gaussian setting,
effectively reduces to graphical Lasso.

These observations prompt a natural question: \emph{Can GLasso-style
super-structure learning be combined with differentiable causal discovery (DCD)
to yield a more efficient and effective optimization pipeline?} We explore this
by applying GLasso as a pre-processing step for NOTEARS~\cite{zheng2018dags},
and observe that, while incurring an 8.2\% drop in F1 score, it achieves a
substantial 78.7\% reduction in runtime. Upon closer examination, the
performance degradation is largely attributable to GLasso's inability to
accurately estimate the precision matrix in high-dimensional regimes (e.g., $d
\geq 100$). Moreover, real-world datasets often contain latent confounders,
further diminishing its utility in guiding scalable differentiable causal
discovery.

These limitations lead us to a more ambitious question: \emph{Can we design a
super-structure learning method that remains robust in high-dimensional settings
with latent confounders, while still being efficient enough to guide
differentiable causal discovery?}

To address this, we propose \tool (\textbf{A}ugmented
\textbf{L}atent-\textbf{V}ariable \textbf{G}raphical \textbf{L}asso), a novel
enhancement to differentiable causal discovery for linear Gaussian data, with or
without latent confounders, by efficiently learning a super-structure to guide
optimization and reduce the search space. \tool employs a sparse+low-rank
decomposition of the precision matrix, inspired by the seminal framework of
Chandrasekaran et al.~\cite{chandrasekaran2012latent}, where the sparse
component captures direct dependencies among observed variables and the low-rank
component absorbs dense correlations from latent confounders or shared variance,
ensuring robustness in high-dimensional settings unlike vanilla Graphical Lasso.
To solve this decomposition efficiently, \tool employs an optimization strategy
based on the Alternating Direction Method of Multipliers (ADMM), ensuring
scalable and stable convergence even in high-dimensional scenarios. Finally,
\tool introduces a robust super-structure learning method that combines
information from sparse and low-rank components through a weighted adjacency
matrix to accurately identify true causal connections while being aware of
associations arising from latent confounding or limited sample sizes. Together,
these optimizations enable \tool to consistently deliver accurate, efficient,
and reliable super-structures for downstream differentiable causal discovery
pipelines.

\begin{figure}
    \centering
    \includegraphics[width=0.85\linewidth]{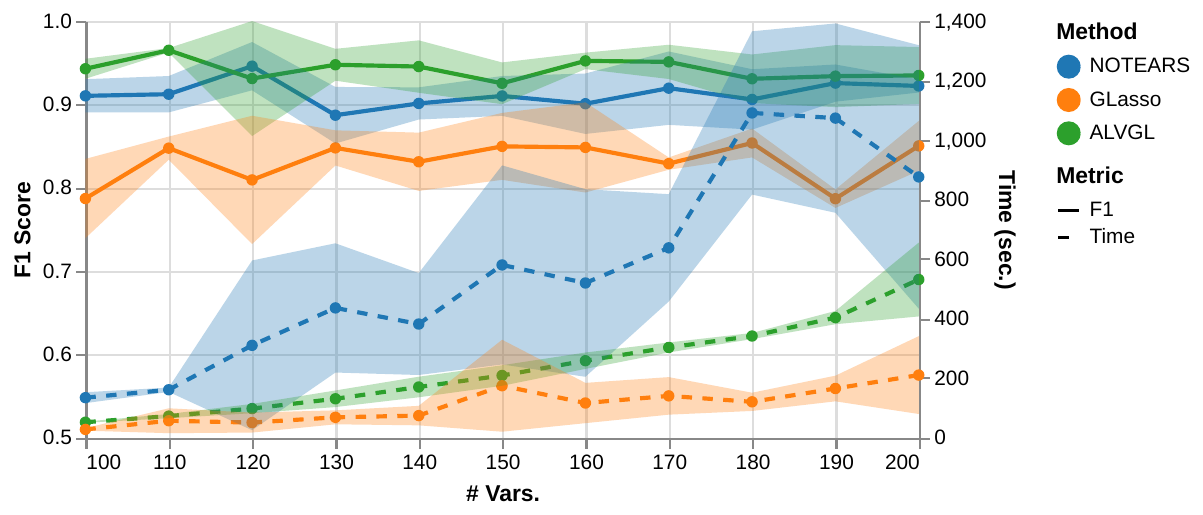}
    \caption{Comparison of NOTEARS, GLasso and \tool on synthetic datasets
    ($n=1000$, degree$=1$, repeat three times).}
    \label{fig:intro:eval}
\end{figure}

\parh{Evaluation Highlights.}~In \F\ref{fig:intro:eval}, we illustrate the
performance advantages of \tool compared to NOTEARS~\cite{zheng2018dags} across
multiple datasets. \tool consistently achieves significant improvements,
boosting the F1 score by an average of 3.3\% across all datasets while
simultaneously reducing runtime by 52.9\%. Notably, applying GLasso to these
same datasets results in failure for 25 out of 33 scenarios due to
ill-conditioned covariance matrices. Although the accuracy gain compared to
NOTEARS appears modest, it is particularly noteworthy in the more challenging
scenario involving latent confounders, where \tool boosts outperforms
ABIC~\cite{bhattacharya2021differentiable}, improving F1 scores by 3.1\% and
runtime by 77.4\%.

In summary, we make the following contributions:
\begin{itemize}
\item We propose a general framework for super-structure-guided differentiable
causal discovery, which significantly improves scalability and reliability in
challenging scenarios involving high-dimensional data and latent confounders.

\item We introduce \tool, a new super-structure learning method, which
decomposes the observed precision matrix into sparse and low-rank components.
This allows it to robustly capture both direct dependencies and latent
confounding effects while maintaining stability in high dimensions. We further
show that the learned super-structure from \tool is theoretically guaranteed to
contain the true causal graph and can be effectively integrated with a variety
of existing differentiable causal discovery methods.

\item Through extensive experiments on synthetic and real-world datasets, we
demonstrate that \tool consistently improves both accuracy and efficiency,
outperforming state-of-the-art baselines across a wide range of settings,
especially in the presence of latent confounders.

\end{itemize}

\parh{Availability.}~We commit to open-sourcing \tool upon acceptance to foster
transparency and facilitate further research. \section{Preliminary}

In accordance with many previous works~\cite{ng2021reliable, tsirlis2018scoring,
chen2021integer, claassen2022greedy, bhattacharya2021differentiable}, we focus
on discovering causal relationships of a linear Gaussian structural causal model
(SCM) with or without latent confounders. In this section, we introduce the
preliminary knowledge of SCM, differentiable causal discovery, and graphical
Lasso.

\subsection{Linear Gaussian SCM}
\label{subsec:linear-scm}

\parh{Latent Confounder-Free SCM.}~We first start with the definition of a
linear Gaussian SCM without latent confounders. Consider a linear SCM with $d$
observable variables parameterized by a coefficient matrix
$\delta\in\mathbb{R}^{d\times d}$. The SCM can be written as
\begin{equation}
    \label{eq:linear-scm}
    X_i\leftarrow \bm{B}_i^T \bm{X} + N_i, \quad i=1,\ldots,d
\end{equation}
where $X_i$ is the $i$-th variable, $\bm{B}_i$ is a vector of coefficients for
the parents of $X_i$ in the DAG representation of the SCM, $\bm{X}$ is a vector
of all variables, and $N_i$ is a noise term that is mutually independent of all
other noise terms. Specifically, a non-zero coefficient $\bm{B}_{ij}$ indicates
that $X_j$ is a parent of $X_i$ in the DAG. \eqref{eq:linear-scm} can be
rewritten in matrix form as
\begin{equation}
    \label{eq:linear-scm-matrix}
    \bm{X} = \bm{B}^T \bm{X} + \bm{N}
\end{equation}
where $\bm{B}=[\bm{B}_1,\ldots,\bm{B}_d]$ corresponds to a weighted adjacency
matrix of the DAG, $\bm{X}=[X_1,\ldots,X_d]^T$ is a vector of all variables, and
$\bm{N}=[N_1,\ldots,N_d]$ is a vector of noise terms. The noise term $\bm{N}$ is
characterized by a multi-variate Gaussian distribution with zero mean and
covariance matrix
$\Omega_{\bm{N}}=\text{cov}(\bm{N})=\text{diag}(\sigma_1^2,\ldots,\sigma_d^2)$,
where $\sigma_i^2$ is the variance of the noise term $N_i$. We assume that
$\sigma_i^2>0$ for all $i=1,\ldots,d$ to ensure positive measure everywhere. The
covariance matrix $\Omega_{\bm{X}}$ of the variables $\bm{X}$ can be expressed
as $\Omega_{\bm{X}}=(I-\bm{B})^{-T}\Omega_{\bm{N}}(I-\bm{B})^{-1}$, where $I$ is
the identity matrix. The joint distribution of $\bm{X}$ is a zero-mean
multivariate Gaussian distribution with covariance matrix $\Omega_{\bm{X}}$.

\parh{Latent Confounder SCM.}~When latent confounders are present, the linear
SCM can be extended to include latent variables $\bm{X}_L$. In this case, the
latent variables can be encoded via the noise term $N_i$ in
\eqref{eq:linear-scm}, which is no longer mutually independent of other noise
terms due to confounding effects. The covariance matrix $\Omega_{\bm{N}}$ of the
noise term $\bm{N}$ is no longer diagonal, and
$\Omega_{\bm{N}_{ij}}=\text{cov}(N_i,N_j)$ is non-zero if and only if $X_i$ and
$X_j$ share a latent confounder. With this extension, the linear SCM on observable
variables $\bm{X}_O$ can still be expressed in the similar form as
\begin{equation}
    \label{eq:linear-scm-latent}
    \bm{X}_O = \bm{B}^T \bm{X}_O + \bm{N}
\end{equation}
And, likewise, the covariance matrix of the observable variables $\bm{X}_O$ can
be expressed as $\Omega_{\bm{X}_O} =(I-\bm{B})^{-T} \Omega_{\bm{N}}
(I-\bm{B})^{-1}$.

\parh{Graphical Representation.}~The SCM defined above can be represented as a
causal graph $G=(\bm{V},\bm{E})$, where $\bm{V}=\{V_1,\ldots,V_d\}$ is a set of
nodes representing the variables $\bm{X}$ (or, in the case of latent
confounders, $\bm{X}_O$), and $\bm{E}$ is a set of directed (and bidirected, if
latent confounders are present) edges representing the causal relationships
between the variables. Under the \textit{structural minimality}
assumption~\cite{peters2017elements}, one can derive the causal graph from the
linear SCM with the following rules: (1) if $\bm{B}_{ij}\neq 0$, then there is a
directed edge $V_i\to V_j$ in the graph; (2) if latent confounders are present,
then there is a bidirected edge $V_i\leftrightarrow V_j$ if
$\Omega_{\bm{N}_{ij}}\neq 0$. Thus, the goal of causal discovery recasts to
inferring $\bm{B}$ and $\Omega_{\bm{N}}$ from the empirical observations of
$\bm{X}$ (or $\bm{X}_O$).

\subsection{Differentiable Causal Discovery}

Score-based methods aim to maximize a score function (e.g., log-likelihood) over
graph structures, subject to the acyclicity constraint. This can be formulated
as:
\begin{equation}
\arg\max_{{G}} f(G) \text{ s.t. } {G} \text{~is acyclic}
\end{equation}
\noindent where $f(\cdot)$ is the score function. Due to the acyclicity
constraint, this optimization is inherently combinatorial. Recently,
differentiable approaches have relaxed this problem by reformulating the
acyclicity constraint into a continuous function --- most notably:
\begin{equation}
\label{eq:dag-constraint}
h_{\text{DAG}}(\bm{B}) = \text{tr}(e^{\bm{B} \circ \bm{B}}) - d
\end{equation}
\noindent where $d = |\bm{V}\text{O}|$ is the number of observed variables, and
$\bm{B}$ is the weighted adjacency matrix of the graph $G$, as defined in
\S~\ref{subsec:linear-scm}. This function satisfies $h_{\text{DAG}}(\bm{B}) = 0
$ if and only if $G$ is acyclic. Several variants of $h_{\text{DAG}}$ have been
proposed to improve numerical stability and computational
efficiency~\cite{yu2019dag, wei2020dags, zheng2020learning, bello2022dagma,
zhang2025analytic}. Typically, these formulations are optimized using augmented
Lagrangian methods (ALM) or the alternating direction method of multipliers
(ADMM), which decompose the constrained problem into a series of unconstrained
subproblems solvable via standard optimizers.

Despite these advances, the optimization remains challenging. The objective
function is often complex, and the acyclicity constraint $h_{\text{DAG}}$
introduces significant non-convexity. In particular, $\bm{B}$ scales
quadratically with the number of variables, and the constraint $h_{\text{DAG}}$
has at least polynomial degree in $\bm{B}$, depending on its specific form.
Consequently, differentiable methods often struggle to converge to meaningful
solutions in large-scale settings.

The scalability issues are exacerbated when latent confounders are present, as
both the objective function and the constraints become more complex. As
aforementioned, the causal graph $G$ requires both $\bm{B}$ and
$\Omega_{\bm{N}}$ to represent the directed and bidirected edges, respectively.
Since bidirected edges are introduced, the graph is represented as an ancestral
directed mixed graph (ADMG)~\cite{richardson2002ancestral}. The corresponding
structural constraint for ADMGs~\cite{bhattacharya2021differentiable} can be
expressed as
\begin{equation}
    \label{eq:admg-constraint}
    h_{\text{ADMG}}(\bm{B},\Omega_{\bm{N}}) = \operatorname{tr}(e^{\bm{B}}) - d +
    \operatorname{sum}(e^{\bm{B}}\circ \Omega_{\bm{N}})
\end{equation}
The function $h_{\text{ADMG}}$ is a generalization of $h_{\text{DAG}}$ to
accommodate the bidirected edges in ADMGs. It has been proved that
$h_{\text{ADMG}}(\bm{B},\Omega_{\bm{N}})=0$ if and only if $G$ is an ancestral
ADMG. And, additional complexity is also introduced in the objective function
for differentiable ADMG learning. We omit the details of the objective function
here as it is independent of our interests in this paper, and refer readers to
the original paper~\cite{bhattacharya2021differentiable}.

\subsection{Graphical Lasso}
\label{subsec:glasso}

Graphical Lasso~\cite{friedman2008sparse} is a widely used method for estimating
the precision matrix (the inverse of the covariance matrix) of a multivariate
Gaussian distribution. We state the problem as follows. Consider a set of
variables $\bm{X} = [X_1,\ldots,X_d]^T$ from a zero-mean multivariate Gaussian
distribution with a covariance matrix $\Omega_{\bm{X}} =
\operatorname{cov}(\bm{X})$. The graphical Lasso aims to estimate the precision
matrix $\Theta$ which is the inverse of covariance matrix $\Omega_{\bm{X}}^{-1}$
by solving the following optimization problem:
\begin{equation}
    \label{eq:glasso}
    \arg\min_{\Theta\succeq 0} -\log\det(\Theta) + \operatorname{tr}(\Omega_{\bm{X}}\Theta) +
    \lambda \|\Theta\|_1
\end{equation}
\noindent where $\lambda$ is a regularization parameter that controls the
sparsity of the estimated precision matrix. The first two terms in the objective
function are the negative log-likelihood of the Gaussian distribution, and the
third term is an $\ell_1$-norm regularization term that encourages sparsity in
the precision matrix $\Theta$. An important property of the precision matrix is
that if $\Theta_{ij}=0$ and $\bm{X}$ is linear Gaussian, then $X_i$ and $X_j$
are conditionally independent given all other variables, i.e., $\bm{X} \setminus
\{X_i, X_j\}$. The undirected graph corresponding to the precision matrix
$\Theta$ is known as the \textit{conditional independence
graph}~\cite{loh2014high}, which is \textit{the moralized DAG} for the setting
without latent confounders. Since \textit{the moralized DAG} is a super graph of
the true causal graph, graphical Lasso can be used as a prescreening step to
reduce the search space for causal discovery
methods~\cite{loh2014high,buhlmann2014cam,ng2021reliable}.

While graphical lasso has proven effective in many applications, it faces
several limitations when applied to high-dimensional data and complex dependency
structures. First, the method assumes that all conditional dependencies can be
captured through direct pairwise relationships in the precision matrix, failing
to account for \textit{latent confounders} that may induce spurious correlations
between observed variables~\cite{chandrasekaran2012latent}. Second, graphical
lasso relies on a SPD covariance matrix $\Omega_{\bm{X}}$ to ensure the
resulting optimization problem is well-posed. However, under high-dimensional
settings where the sample size $n$ is not significantly larger than the number
of variables $d$ (i.e., $n \not\gg d$), or when the underlying graph has dense
structures or subtle noises (i.e., nearly linearly dependent variables), the
resulting matrix may become rank-deficient during optimizations, leading to
ill-conditioned problems. 

\begin{figure*}
    \centering
    \includegraphics[width=0.99\linewidth]{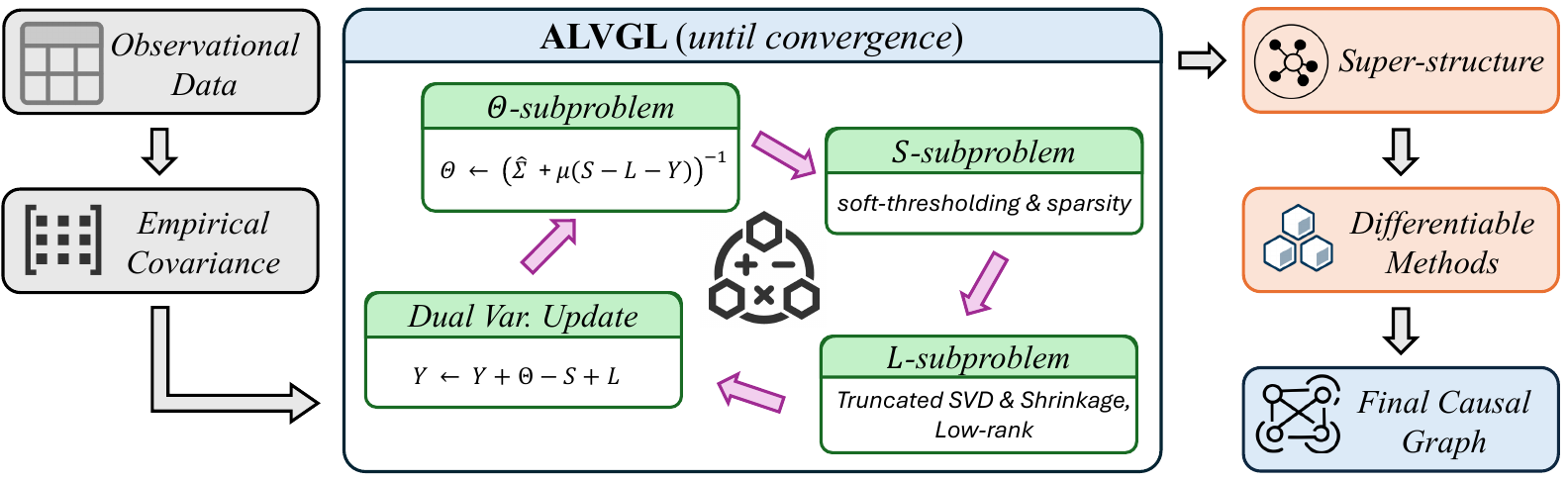}
    \caption{Overview}
    \label{fig:intro:overview}
\end{figure*}

\section{\tool}
\label{sec:dlvgl}

Inspired by the success of super-structure-guided optimization in score-based
methods, we are motivated to improve the limitations of graphical Lasso and
learn a super-structure that can better guide the optimization of differentiable
causal discovery methods. We start with formulating the problem setting and then
present our approach.

\subsection{Problem Setting}
\label{subsec:setting}

We consider the task of causal structure discovery in the presence of latent
confounders under the linear Gaussian assumption. Let $\bm{X}_O = [X_1, \dots,
X_d]^T \in \mathbb{R}^{d \times n}$ denote a dataset of $d$ observed variables
with $n$ samples, where each $X_i$ is a real-valued random variable, and the
observations are generated from a linear Gaussian SCM that may include latent
variables. The goal is to learn a \emph{super-structure} --- an undirected graph
over $\bm{X}_O$ that contains all true edges of the underlying causal graph ---
to restrict and guide the search space for downstream causal discovery
algorithms.

We assume the data follows a linear SCM with additive Gaussian noise:
\begin{equation}
    \bm{X} = \bm{B}^T \bm{X} + \bm{N}
\end{equation}
where $\bm{X} \in \mathbb{R}^{(d+\ell) \times n}$ consists of both observed
variables $\bm{X}_O$ and latent variables $\bm{X}_L$, $\bm{B}$ is a coefficient
matrix encoding the causal relationships (possibly including connections to
latent variables), and $\bm{N} \sim \mathcal{N}(0, \Omega_{\bm{N}})$ is a
zero-mean Gaussian noise vector. We let $\Sigma_O = \text{cov}(\bm{X}_O)$ denote
the marginal covariance matrix over observed variables, and our aim is to
recover structural information about the true causal graph using only access to
$\bm{X}_O$. Specifically, when $l=0$, the true causal graph is a DAG over
$\bm{X}_O$; when $l>0$, the true causal graph is a maximal ancestral graph (MAG)
over $\bm{X}_O$ that may include both directed and bidirected edges due to
latent confounding~\cite{richardson2002ancestral}.

In the following, we make the standard assumptions commonly adopted in the
causal discovery literature~\cite{bhattacharya2021differentiable,
peters2017elements} and forms the basis of our theoretical analysis.

\begin{enumerate}[label=(A\arabic*)]
  \item \textbf{Linear Gaussian SCM with latent variables.} The joint
  distribution of $(\bm{X}_O,\bm{X}_L)$ is generated by a linear Gaussian SCM of
  the form above whose causal structure is represented by a directed acyclic
  graph (DAG) over observed and latent variables, possibly inducing a mixed
  graph over $\bm{X}_O$ after marginalizing out latents.

  \item \textbf{Faithfulness and structural minimality.} The joint distribution
  is faithful and structurally minimal with respect to the underlying causal
  graph~\cite{peters2017elements}. In particular, for DAGs and mixed graphs,
  missing edges (or graphical separations in general) correspond to conditional
  independencies, and there is no exact cancellation of parameters that would
  mask an existing edge (i.e., by manifesting a detectable statistical
  dependence \textit{w.r.t} a threshold $\tau$).
\end{enumerate}

\subsection{On the Implication of Latent Variables}
\label{subsec:latent}

A major challenge in causal discovery with latent variables is that
marginalizing out these unobserved variables distorts the conditional
independence structure among the observed ones. In particular, the marginal
precision matrix $\Theta_O = \Sigma_O^{-1}$, where $\Sigma_O$ is the covariance
matrix over the observed variables $\bm{X}_O$, no longer reflects a sparse
structure, even when the true causal graph is sparse. This violates the core
assumption of methods like graphical Lasso, which rely on the sparsity of
$\Theta_O$ to recover meaningful structural information.

To understand the issue, we distinguish between two ways of interpreting the
precision matrix. On one hand, $\Theta_O$ encodes the conditional independencies
among observed variables given only other observed variables, since it is
computed as the inverse of the marginal covariance $\Sigma_O$. On the other
hand, the correct conditional independence structure that aligns with the full
causal model, including latent variables, is represented by the submatrix
$\tilde{\Theta}_O = \Theta_{O+L}[O,O]$ of the full precision matrix
$\Theta_{O+L} = \Sigma_{O+L}^{-1}$. In this case, $\tilde{\Theta}_O$ reflects
conditional independencies among observables given all remaining observed and
latent variables. The two matrices differ: while $\Theta_{O_{ij}} = 0$ implies
$X_i$ and $X_j$ are independent conditioned only on other observables,
$\tilde{\Theta}_{O_{ij}} = 0$ implies independence even after accounting for
latent confounding.

In the context of causal discovery, one might hope that $\Theta_O$ is sufficient
to determine the super-structure, since an edge in the true causal graph implies
that $X_i$ and $X_j$ are not conditionally independent given any subset of the
other observed variables. This would suggest that directly applying graphical
Lasso to estimate $\Theta_O$ from $\Sigma_O$ could suffice. However, this
intuition is misleading. Marginalizing over latent variables causes their
influence to spread across the observed space, and the resulting precision
matrix $\Theta_O$ becomes dense, with indirect effects falsely appearing as
direct ones. The sparsity assumption on $\Theta_O$ no longer holds, making
graphical Lasso inapplicable in this setting~\cite{chandrasekaran2012latent}.
Therefore, it is crucial to account for latent confounding explicitly when
estimating the precision matrix. A better alternative is to model the observed
conditional independencies via a decomposition of the full precision structure
into a sparse component, capturing direct effects among observables, and a
low-rank component, capturing dependencies induced by latent variables. We turn
to this approach in the next section.

\subsection{Sparse+Low-Rank Decomposition}
\label{subsec:splr}

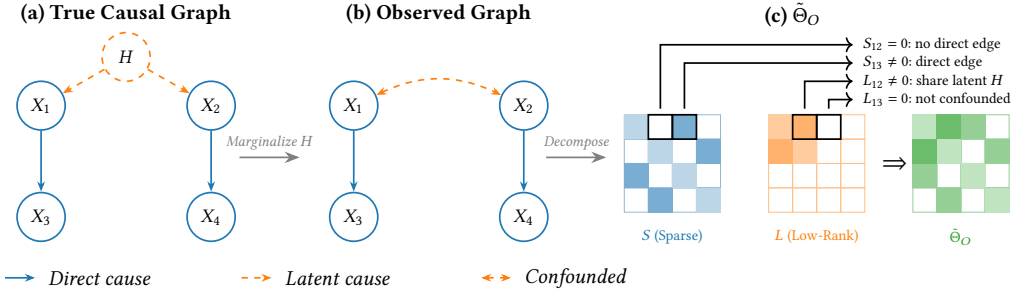
\begin{figure}[t]
\centering
\resizebox{0.98\linewidth}{!}{
\begin{tikzpicture}[
    node distance=1.2cm,
    obs/.style={circle, draw=vegaBlue, thick, minimum size=0.8cm, font=\small},
    latent/.style={circle, draw=vegaOrange, thick, dashed, minimum size=0.8cm, font=\small},
    direct/.style={-{Stealth[length=1.7mm]}, thick, vegaBlue},
    confound/.style={{Stealth[length=1.7mm]}-{Stealth[length=1.7mm]}, thick, vegaOrange, dashed},
    latedge/.style={-{Stealth[length=1.7mm]}, thick, vegaOrange, dashed},
    annotation/.style={font=\scriptsize, align=center},
    cellbox/.style={draw, thick, minimum width=0.5cm, minimum height=0.5cm, font=\scriptsize},
    explain/.style={font=\scriptsize, align=left}
]

\begin{scope}[local bounding box=panelA]

\coordinate (A) at (0,0);

\node[obs] (X1) at (A) {$X_1$};
    \node[obs, right=2cm of X1] (X2) {$X_2$};
    \node[obs, below=1cm of X1] (X3) {$X_3$};
    \node[obs, below=1cm of X2] (X4) {$X_4$};

\node[latent] (L) at ($(X1)!0.5!(X2)+(0,0.8cm)$) {$H$};

\node[annotation, font=\bfseries] at ($(X1)!0.5!(X2)+(0,1.5cm)$) {(a) True Causal Graph};

\draw[direct] (X1) -- (X3);
    \draw[direct] (X2) -- (X4);

\draw[latedge] (L) -- (X1);
    \draw[latedge] (L) -- (X2);

\node[annotation, anchor=west] (leg1) at ($(X3)+(0,-1.0cm)$) {\it\normalsize Direct cause};
    \draw[direct] ($(leg1.west)+(-0.6cm,0)$) -- ($(leg1.west)+(-0.1cm,0)$);

    \node[annotation, anchor=west, right=2.0cm of leg1] (leg2) {\it\normalsize Latent cause};
    \draw[latedge] ($(leg2.west)+(-0.6cm,0)$) -- ($(leg2.west)+(-0.1cm,0)$);

    \node[annotation, anchor=west, right=2.0cm of leg2] (leg3) {\it\normalsize Confounded};
    \draw[confound] ($(leg3.west)+(-0.6cm,0)$) -- ($(leg3.west)+(-0.1cm,0)$);

\end{scope}

\draw[-{Stealth[length=2.5mm]}, thick, gray] (3.3, -0.85) -- (4.3, -0.85);
\node[annotation, above, gray] at (3.8, -0.85) {\it Marginalize $H$};

\begin{scope}[xshift=5.2cm, local bounding box=panelB]

\coordinate (B) at (0,0);

\node[obs] (Y1) at (B) {$X_1$};
    \node[obs, right=2cm of Y1] (Y2) {$X_2$};
    \node[obs, below=1cm of Y1] (Y3) {$X_3$};
    \node[obs, below=1cm of Y2] (Y4) {$X_4$};

\node[annotation, font=\bfseries] (Blabel)
        at ($(Y1)!0.5!(Y2)+(0,1.5cm)$) {(b) Observed Graph};

\draw[direct] (Y1) -- (Y3);
    \draw[direct] (Y2) -- (Y4);

\draw[confound] (Y1) to[bend left=25] (Y2);

\end{scope}

\draw[-{Stealth[length=2.5mm]}, thick, gray] (8.4, -0.85) -- (9.4, -0.85);
\node[annotation, above, gray] at (8.9, -0.85) {\it Decompose};

\begin{scope}[xshift=10.5cm, local bounding box=panelC]
\coordinate (C) at (0,0);
    \node[annotation, font=\bfseries]
    at ($(C |- Blabel)+(2.0cm,0)$) {(c) $\tilde{\Theta}_O$};

\coordinate (Sbase) at ($(C) + (-0.8,-0.15)$);
    \foreach \i in {0,1,2,3}{
        \foreach \j in {0,1,2,3}{
            \draw[vegaBlue!50]
            ($(Sbase)+(\j*0.4, -\i*0.4)$)
            rectangle
            ($(Sbase)+(+0.4+\j*0.4, -0.4-\i*0.4)$);
        }
    }

\fill[vegaBlue!30] (-0.8,-0.15) rectangle (-0.4,-0.55);
    \fill[vegaBlue!30] (-0.4,-0.55) rectangle (0,-0.95);
    \fill[vegaBlue!30] (0,-0.95) rectangle (0.4,-1.35);
    \fill[vegaBlue!30] (0.4,-1.35) rectangle (0.8,-1.75);

    \fill[vegaBlue!60] (-0.8,-0.95) rectangle (-0.4,-1.35);
    \fill[vegaBlue!60] (0,-0.15) rectangle (0.4,-0.55);
    \fill[vegaBlue!60] (-0.4,-1.35) rectangle (0,-1.75);
    \fill[vegaBlue!60] (0.4,-0.55) rectangle (0.8,-0.95);

\draw[black, thick] (-0.4,-0.15) rectangle (0,-0.55);
    \draw[black, thick, ->] (-0.2,-0.05) -- (-0.2,1.0) -- (3,1.0);
    \node[explain,anchor=west] at (3.05,1.0) {$S_{12}=0$: no direct edge};

    \draw[black, thick] (0,-0.15) rectangle (0.4,-0.55);
    \draw[black, thick, ->] (0.2,-0.05) -- (0.2,0.70) -- (3,0.70);
    \node[explain,anchor=west] at (3.05,0.70) {$S_{13}\neq0$: direct edge};

    \node[annotation,text=vegaBlue] at (0,-2.15) {$S$ (Sparse)};

\coordinate (Lbase) at ($(C)+(1.6,-0.15)$);
    \foreach \i in {0,1,2,3}{
        \foreach \j in {0,1,2,3}{
            \draw[vegaOrange!50] ($(Lbase)+(\j*0.4, -\i*0.4)$)
            rectangle
            ($(Lbase)+(+0.4+\j*0.4, -0.4-\i*0.4)$);
        }
    }
    \begin{scope}[yshift=-0.75cm]
    \fill[vegaOrange!40] (1.6,0.6) rectangle (2.0,0.2);
    \fill[vegaOrange!40] (2.0,0.2) rectangle (2.4,-0.2);
    \fill[vegaOrange!60] (2.0,0.6) rectangle (2.4,0.2);
    \fill[vegaOrange!60] (1.6,0.2) rectangle (2.0,-0.2);

\draw[black, thick] (2.4,0.6) rectangle (2.8,0.2);
    \draw[black, thick, ->] (2.6,0.7) -- (2.6,0.85) -- (3,0.85);
    \node[explain,anchor=west] at (3.05,0.85)
        {$L_{13}=0$: not confounded};

\draw[black, thick] (2.0,0.6) rectangle (2.4,0.2);
    \draw[black, thick, ->] (2.2,0.7) -- (2.2,1.15) -- (3,1.15);
    \node[explain,anchor=west] at (3.05,1.15)
        {$L_{12}\neq0$: share latent $H$};

    \node[annotation,text=vegaOrange] at (2.4,-1.4) {$L$ (Low-Rank)};
    \node[font=\Large] at (3.7,-0.2) {$\Rightarrow$};
    \end{scope}

\coordinate (Thetabase) at ($(C)+(4.0,-0.15)$);
    \foreach \i in {0,1,2,3}{
        \foreach \j in {0,1,2,3}{
            \draw[vegaGreen!50] ($(Thetabase)+(\j*0.4, -\i*0.4)$)
                 rectangle ($(Thetabase)+(0.4+\j*0.4, -0.4-\i*0.4)$);
        }
    }
    \begin{scope}[yshift=-0.75cm]
    \fill[vegaGreen!30] (4.0,0.6) rectangle (4.4,0.2);
    \fill[vegaGreen!30] (4.4,0.2) rectangle (4.8,-0.2);
    \fill[vegaGreen!30] (4.8,-0.2) rectangle (5.2,-0.6);
    \fill[vegaGreen!30] (5.2,-0.6) rectangle (5.6,-1.0);
    \fill[vegaGreen!50] (4.0,-0.2) rectangle (4.4,-0.6);
    \fill[vegaGreen!50] (4.8,0.6) rectangle (5.2,0.2);
    \fill[vegaGreen!50] (4.4,-0.6) rectangle (4.8,-1.0);
    \fill[vegaGreen!50] (5.2,0.2) rectangle (5.6,-0.2);
    \fill[vegaGreen!70] (4.4,0.6) rectangle (4.8,0.2);
    \fill[vegaGreen!70] (4.0,0.2) rectangle (4.4,-0.2);

    \node[annotation,text=vegaGreen] at (4.8,-1.4) {$\tilde{\Theta}_O$};
    \end{scope}
\end{scope}

\end{tikzpicture}
} 

\caption{Illustration of sparse+low-rank decomposition. (a) The true causal
graph includes a latent confounder $H$ affecting $X_1$ and $X_2$, with direct
edges $X_1 \to X_3$ and $X_2 \to X_4$. (b) After marginalizing $H$, its effect
appears as a bidirected edge $X_1 \leftrightarrow X_2$. (c) The precision matrix
$\tilde{\Theta}_O$ decomposes into sparse $S$ and low-rank $L$. $S_{ij} \neq 0$
indicates a direct causal relationship between $X_i$ and $X_j$; $L_{ij} \neq 0$
indicates that $X_i$ and $X_j$ are jointly influenced by shared latent
variables; $\tilde{\Theta}_{O,ij}$ combines both effects to represent the full
conditional dependency structure.}
\label{fig:splr-decomposition}
\end{figure} 
\F~\ref{fig:splr-decomposition} provides an intuitive illustration of how a
precision matrix with latent confounding effects naturally separates into a
\emph{sparse} component $S$ capturing direct causal relationships among observed
variables and a \emph{low-rank} component $L$ capturing shared variation induced
by latent factors. Panel~(a) shows the true causal graph with a latent variable
$H$ influencing $X_1$ and $X_2$. After marginalizing $H$, its effect appears in
panel~(b) as a bidirected edge $X_1 \leftrightarrow X_2$, which is a hallmark of
latent confounding. Panel~(c) then visualizes the decomposition
$\tilde{\Theta}_O = S - L$, where annotated matrix cells illustrate that nonzero
entries in $S$ (blue blocks) correspond to direct causal relations (e.g.,
$S_{13} \neq 0$), while nonzero entries in $L$ (orange blocks) correspond to
latent-induced correlations (e.g., $L_{12} \neq 0$ because $X_1$ and $X_2$ share
$H$), and the final precision $\tilde{\Theta}_O$ reflects their combined
effects.

\medskip
Recall that the observed precision matrix $\tilde{\Theta}_O$ consists of
\ding{192} $\Theta_O$, the conditional dependencies among observed variables, and
\ding{193} a low-rank component encoding the conditional dependencies induced by
latent variables. Formally, following \citet{chandrasekaran2012latent}, the
precision matrix after marginalizing latent variables can be expressed via the
Schur complement as:
\begin{equation}
    \tilde{\Theta}_O
    = \underbrace{\Theta_O}_{S}
    - \underbrace{\Theta_{OL} \Theta_L^{-1} \Theta_{LO}}_{L},
\label{eq:schur}
\end{equation}
where $\Theta_{OL}$ is the cross-precision between observed and latent
variables, and $\Theta_L$ is the precision of the latent variables.

\paragraph{Connection to \F~\ref{fig:splr-decomposition}.} In panel~(a), the
latent variable $H$ induces correlations among $X_1$ and $X_2$. This translates
algebraically into the low-rank term $L = \Theta_{OL} \Theta_L^{-1} \Theta_{LO}$
in \eqref{eq:schur}, which appears in panel~(c) as highlighted orange cells such
as $L_{12} \neq 0$. Conversely, direct causal edges such as $X_1 \to X_3$
manifest as nonzero entries in $S$, e.g., $S_{13} \neq 0$ (blue cells). The
decomposition thus separates the effects of direct causation and latent
confounding into interpretable structural components.

To estimate the decomposition, \citet{chandrasekaran2012latent} proposed the
convex program:
\begin{align}
\min_{S-L\succ 0,\, L \succeq 0}\;
\mathcal{L}(S, L)
&=
- \log\det(S - L)
+ \operatorname{tr}[(S - L)\hat{\Sigma}]
\nonumber\\
&\quad
+ \lambda_s \|S\|_1 + \lambda_l \|L\|_*,
\label{eq:objective}
\end{align}
where $\|S\|_1$ enforces sparsity to recover direct dependencies (blue blocks in
panel~(c)), and $\|L\|_*$, the nuclear norm, enforces low-rank structure
matching the shared latent effects (orange blocks in panel~(c)).
The resulting estimate of $\Theta_O$ is thus recovered as $S$, while $L$
provides an interpretable summary of latent confounding patterns. As shown in
panel~(c), $S$ and $L$ jointly reconstruct $\tilde{\Theta}_O$ (green matrix),
demonstrating how direct and latent-induced dependencies combine.

When no latent confounders exist, $L$ is expected to be zero and
$\tilde{\Theta}_O = \Theta_O$. Even in this setting, the sparse+low-rank model
is empirically more robust than graphical Lasso, particularly when $n$ is not
significantly larger than $d$. The low-rank component naturally absorbs dense,
spurious correlations caused by small sample sizes or model misspecification,
preventing contamination of the sparse structure in $S$. This leads to improved
conditioning and more stable convergence of the precision estimation problem.

\subsection{ADMM Optimization}
\label{subsec:optim}

To solve the optimization problem stated in \eqref{eq:objective}, we first
present a reformulation of the problem that allows us to use the Alternating
Direction Method of Multipliers (ADMM)~\cite{boyd2011distributed}. We introduce
an auxiliary variable $\Theta = S - L$, which leads to the following equivalent
formulation:
\begin{align}
\min_{\Theta \succ 0,\ L \succeq 0,\ S} & -\log\det(\Theta) + \operatorname{tr}(\Theta\hat{\Sigma}) + \lambda_s \|S\|_1 + \lambda_l \|L\|_*\\
\text{subject to } &\Theta = S - L
\label{eq:reformulated-objective}
\end{align}
The equality constraint $\Theta = S - L$ is where ADMM comes in. We have the
following augmented Lagrangian function:
\begin{align}
\mathcal{L}_\mu(\Theta, S, L, Y)
= & -\log\det(\Theta) + \operatorname{tr}(\Theta \hat{\Sigma})+ \lambda_s \|S\|_1\\
&  + \lambda_l \|L\|_* + \frac{\mu}{2} \| \Theta - S + L + Y \|_F^2.
\label{eq:augmented-lagrangian}
\end{align}
where the residual $R \coloneq \Theta - (S - L)$ is captured by the dual
variable $Y$. We have the subproblem updates as follows
\begin{equation}  
\begin{aligned}
\Theta^{k+1}&=\arg\min_{\Theta\succ0}
   -\!\log\det\Theta+\operatorname{tr}(\Theta\hat\Sigma)
   +\tfrac{\mu}{2}\!\left\|\Theta-S^k+L^k+Y^k\right\|_F^{2},\\
S^{k+1}&=\arg\min_S
   \lambda_s\|S\|_1+\tfrac{\mu}{2}\!\left\|\Theta^{k+1}-S+L^k+Y^k\right\|_F^{2},\\
L^{k+1}&=\arg\min_{L\succeq0}
   \lambda_l\|L\|_{*}+\tfrac{\mu}{2}\!\left\|\Theta^{k+1}-S^{k+1}+L+Y^k\right\|_F^{2},\\
Y^{k+1}&=Y^{k}+\Theta^{k+1}-S^{k+1}+L^{k+1}.
\end{aligned}
\end{equation}
By solving these subproblems iteratively, we can obtain the estimates of $S$ and
$L$. We present the operational procedure in \A~\ref{alg:admm}.

\begin{algorithm}[t]
\caption{\tool-ADMM}
\label{alg:admm}
\KwIn{$\hat{\Sigma}$: empirical covariance; $\lambda_s, \lambda_l$: regularization; $\tau$: threshold; $T$: max iterations}
Initialize $(S, L, Y, \mu) \leftarrow (\hat{\Sigma}^{-1}, 0, 0, 1)$\;
\For{$t = 1, \ldots, T$ \textbf{or until converged}}{
    $\Theta \gets \left(\hat{\Sigma} + \mu(S - L - Y)\right)^{-1}$ \tcp*[r]{solve \(\Theta\)-subproblem}
    
    $S \gets \operatorname{soft}_{\frac{\lambda_s}{\mu}}\left(\Theta + L + Y\right)$ \tcp*[r]{solve \(S\)-subproblem}

    $\lambda_{\max} \gets \max(\text{eig}(\hat{\Sigma}))$\;
    $r^\star \gets \sum_{i=1}^{d} \mathbb{I}\left( \lambda_i > \tau \cdot \lambda_{\max} \right)$\;

    $[U, \bm{s}, V^\top] \gets \text{SVD}(\Theta - S + Y)$\; 
    $\bm{s} \gets \operatorname{shrink}(\bm{s}, \frac{\lambda_l}{\mu}, r^\star)$\;
    $L \gets U \cdot \operatorname{diag}(\bm{s})\cdot V^\top$ \tcp*[r]{solve \(L\)-subproblem}

    $Y \gets Y + \Theta - S + L$ \tcp*[r]{dual variable update}
    
    Update $\mu$\;
}
\KwRet{$S$, $L$}
\end{algorithm}

In line 1, \A~\ref{alg:admm} initializes the sparse matrix $S$ as the inverse of
the empirical covariance $\hat{\Sigma}$ (or a regularized version if
$\hat{\Sigma}$ is singular), the low-rank matrix $L$ and the scaled dual
variable $Y$ as zero matrices, and the penalty parameter $\mu$ to an initial
value. It then proceeds through a series of iterative updates for $\Theta$, $S$,
$L$, and $Y$:
\begin{enumerate}
    \item \textbf{$\Theta$-update (line 3):} The precision matrix $\Theta^{k+1}$
    is updated by computing the inverse of the effective covariance matrix
    $(\hat{\Sigma} + \mu^{(k)}(L^{(k)} - Y^{(k)}))$, which is an inexact yet
    computationally cheap ADMM update.

    \item \textbf{$S$-update (line 4):} The sparse component $S$ is updated to
    encourage sparsity while fitting the data, which is done by applying the
    element-wise soft-thresholding operator $\operatorname{soft}_{\alpha}$ to
    $\Theta^{(k+1)}$. Here, $\operatorname{soft}_{\alpha}$ is defined as:
    \begin{equation}
        \operatorname{soft}_{\alpha}(S)_{ij} = \operatorname{sign}(S_{ij})
    \max(|S_{ij}| - \alpha, 0)        
    \end{equation}
    This operation shrinks the elements of
    $\Theta^{(k+1)}$ towards zero, promoting sparsity in the resulting matrix
    $S^{(k+1)}$. The threshold $\alpha$ is set to $\frac{\lambda_s}{\mu^{(k)}}$,
    where $\lambda_s$ is the regularization parameter for sparsity and
    $\mu^{(k)}$ is the penalty parameter at iteration $k$.
    \item \textbf{$L$-update (line 5--9):} The textbook ADMM step for the
    $L$-update requires a full SVD of the residual matrix with soft-thresholding
    to ensure positive semidefiniteness. Recall that $L$ is expected to be
    low-rank in the problem setting, so we pre-estimate a plausible latent
    dimension $r^\star\ll d$ and apply a $\operatorname{shrink}$ operator to
    implement truncated singular value thresholding where only the top $r^\star$
    singular values are retained.\footnote{Computing $r^\star$ is an one-time
    effort and we place it inside the loop for better readability.} Finally, we
    reconstruct the low-rank component $L^{(k+1)}$ using the truncated singular
    values such that the positive semidefiniteness is enforced. Compared to the
    exact solution, this approximation is computationally efficient and
    numerically more robust.
    \item \textbf{$Y$-update (line 10):} The scaled dual variable $Y$ is updated
    using the residual and enforces the equality constraint $\Theta = S - L$
    over the iterations.
\end{enumerate}
These steps are repeated until a convergence criterion is met, or a maximum
number of iterations $T$ is reached. The algorithm then returns the estimated
$S$ and $L$.
\begin{remark}
    \textit{Despite several approximations are introduced in the ADMM updates,
    by the inexact ADMM theory~\cite{boyd2011distributed, eckstein1992douglas},
    \A~\ref{alg:admm} can still converge under some mild conditions.
    Empirically, we find it works better on our task compared to the standard
    ADMM-based latent-variable graphical Lasso
    algorithm~\cite{ma2013alternating}.}
\end{remark}

\noindent\textbf{\tool vs. Classical LVGL.}  Unlike
LVGL~\cite{chandrasekaran2012latent}, which optimizes the decomposition
exclusively for precision estimation, \tool modifies both the penalty structure
and the decomposition objective to maximize support recovery quality of $S$ for
super-structure construction. This design avoids the common leakage of true
edges into the low-rank component, a failure mode that severely limits LVGL when
used for structural tasks, as we will demonstrate in \S~\ref{subsec:rq2}.

\subsection{Super-Structure Learning}
\label{subsec:ssl}

\begin{figure*}[h]
\centering
\resizebox{\textwidth}{!}{\begin{tikzpicture}[
    node distance=1cm and 1.2cm,
    font=\sffamily\small,
var/.style={circle, draw, minimum size=0.5cm, fill=white, thick},
    latent/.style={circle, draw, minimum size=0.5cm, dashed, fill=gray!10},
    arrow/.style={-{Stealth[length=3mm, width=2mm]}, ultra thick, color=gray!80},
    section_label/.style={font=\bfseries, align=center},
    sub_label/.style={font=\footnotesize, align=center, text=gray},
    annotation/.style={font=\scriptsize, align=center},
]
\begin{scope}[local bounding box=panelA]
\node[section_label] (step1_label) {(a) Estimate $S,L$ via \A~\ref{alg:admm}};
    \node[below=0.1cm of step1_label] (SLcontainer) {
    \begin{tikzpicture}[baseline]
\node (SparsePic) {
            \begin{tikzpicture}
                \foreach \i in {0,1,2,3}{
                    \foreach \j in {0,1,2,3}{
                        \draw[vegaBlue!20] (\j*0.4, -\i*0.4) rectangle +(0.4,0.4);
                    }
                }
                \foreach \k in {0,1,2,3}{
                    \fill[vegaBlue!60] (\k*0.4, -\k*0.4) rectangle +(0.4,0.4);
                }
                \fill[vegaBlue!40] (0.4,0.4) rectangle +(0.4,-0.4);
                \fill[vegaBlue!40] (0,-0) rectangle +(0.4,-0.4);
                \fill[vegaBlue!40] (0.4,-0.8) rectangle +(0.4,-0.4);
                \fill[vegaBlue!40] (1.2,0) rectangle +(0.4,-0.4);
            \end{tikzpicture}
        };

        \node[below=0.05cm of SparsePic, font=\footnotesize, text=vegaBlue] (SparseLabel) {Sparse $S$};

\node[right=0.6cm of SparsePic] (LowRankPic) {
            \begin{tikzpicture}
                \foreach \i in {0,1,2,3}{
                    \foreach \j in {0,1,2,3}{
                        \draw[vegaOrange!20] (\j*0.4, -\i*0.4) rectangle +(0.4,0.4);
                    }
                }
                \fill[vegaOrange!60] (0.4,-0) rectangle +(0.4,-0.4);
                \fill[vegaOrange!60] (0.8,-0.4) rectangle +(0.4,-0.4);
                \fill[vegaOrange!60] (1.2,-0.8) rectangle +(0.4,-0.4);

                \fill[vegaOrange!50] (0.8,-0) rectangle +(0.4,-0.4);
                \fill[vegaOrange!50] (0.4,-0.4) rectangle +(0.4,-0.4);

                \fill[vegaOrange!50] (1.2,-0.4) rectangle +(0.4,-0.4);
                \fill[vegaOrange!50] (0.8,-0.8) rectangle +(0.4,-0.4);
            \end{tikzpicture}
        };

        \node[below=0.05cm of LowRankPic, font=\footnotesize, text=vegaOrange] (LowRankLabel) {Low-Rank $L$};

    \end{tikzpicture}
    };
    \end{scope}

\draw[-{Stealth[length=2.5mm]}, thick, gray] (2.3, -1.75) -- (3.9, -1.75);
    \node[annotation, above, gray] at (3.1, -1.75) {\textit{Construct}\\\textit{Super-Structure}};

\begin{scope}[local bounding box=panelB, xshift=6.0cm]
    \node[annotation, font=\bfseries] at (0,0) {(b) Learned Super-Structure};

    \node (SuperGraph) at (0, -1.3cm) {
        \begin{tikzpicture}[scale=0.9]
            \node[var] (x1) at (0, 1.2) {\tiny $X_1$};
            \node[var] (x2) at (1.5, 1.2) {\tiny $X_2$};
            \node[var] (x3) at (0, 0) {\tiny $X_3$};
            \node[var] (x4) at (1.5, 0) {\tiny $X_4$};
            
\draw[thick, vegaBlue] (x1) -- (x2) node[midway, above, font=\tiny] {$S_{12}$};
            \draw[thick, vegaBlue] (x2) -- (x4) node[midway, right, font=\tiny] {$S_{24}$};
            
\draw[thick, vegaOrange, dashed] (x3) -- (x4);
            \draw[thick, vegaOrange, dashed] (x2) -- (x3);
            
\node[latent, scale=0.8, opacity=0.5] (L) at (0.75, -0.4) {$H$};
        \end{tikzpicture}
    };

\node[below=0.05cm of SuperGraph, align=left, font=\scriptsize, fill=white, draw=gray!20, rounded corners, inner sep=3pt] (Legend) {
        \textcolor{vegaBlue}{\textbf{---}} Direct (from $S$)
        \textcolor{vegaOrange}{\textbf{- -}} Confounding (from $L$)
    };

    \end{scope}

\draw[-{Stealth[length=2.5mm]}, thick, gray] (8.2, -1.75) -- (9.8, -1.75);
    \node[annotation, above, gray] at (9, -1.75) {\textit{Initialize \&}\\\textit{Constrain Search}};

\begin{scope}[local bounding box=panelC, xshift=11.6cm]
    \node[annotation, font=\bfseries] at (0,0) {(c) Final MAG (via DCD)};
    \node (FinalGraph) at (0, -1.3cm) {
        \begin{tikzpicture}[scale=0.9]
            \node[var, fill=vegaGreen!10] (f1) at (0, 1.2) {\tiny $X_1$};
            \node[var, fill=vegaGreen!10] (f2) at (1.5, 1.2) {\tiny $X_2$};
            \node[var, fill=vegaGreen!10] (f3) at (0, 0) {\tiny $X_3$};
            \node[var, fill=vegaGreen!10] (f4) at (1.5, 0) {\tiny $X_4$};

\draw[-{Stealth}, thick] (f1) -- (f2);
            \draw[-{Stealth}, thick] (f2) -- (f4);

\draw[{Stealth}-{Stealth}, thick, red] (f3) -- (f4);
        \end{tikzpicture}
    };
    \node[below=0.05cm of FinalGraph, font=\footnotesize, text=gray] {Optimized Causal Graph};
    \end{scope}
\end{tikzpicture}} \caption{Workflow of \tool. \textbf{(a)} The precision matrix
$\Theta$ is decomposed into Sparse ($S$) and Low-Rank ($L$) components.
\textbf{(b)} These components are combined to form a Super-Structure containing
both direct and latent associations. \textbf{(c)} This structure constrains the
search space for Differentiable Causal Discovery (DCD) to efficiently output the
final causal graph (MAG in this case).}
\label{fig:alvgl_reorganized}
\end{figure*}
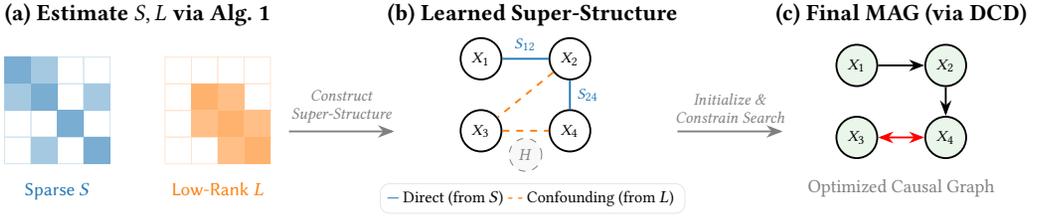

In general, one might expect that the sparse component $S$ (i.e., $\Theta_O$ in
\eqref{eq:schur}) alone would suffice to derive the desired super-structure,
because $\Theta_{O_{ij}} = 0$ implies $X_i$ and $X_j$ are conditionally
independent \emph{given all other observed and latent variables}. However, this
implication highlights a critical caveat:  in the presence of latent
confounders, $\Theta_{O_{ij}} = 0$ \textit{does not} guarantee that $X_i$ and
$X_j$ are conditionally independent given only the other observed variables.
Instead, conditional independence is only assured when conditioning on all other
variables (both observed and latent). As a direct consequence, relying solely on
$S$ may fail to recover all edges present in the true causal graph in the sense
that missing links may arise wherever latent confounding obscures conditional
dependencies among observables. Even if latent variables are absent,
incorporating information from the low-rank component $L$ can still improve
super-structure learning. This is because $L$ not only captures shared variation
induced by latent confounders in an asymptotic regime, but also helps absorb
dense or spurious correlations from finite data samples or model
misspecification in practical settings, thereby enhancing robustness and
stability in estimation.

To further clarify, as schematically illustrated in
\F~\ref{fig:alvgl_reorganized}, our approach computes the combined adjacency
matrix from both components (\F~\ref{fig:alvgl_reorganized}~(a)):
\begin{equation}
W = |S| + |L|
\label{eq:weights}
\end{equation}
The resulting super-structure (\F~\ref{fig:alvgl_reorganized}~(b)) is then
obtained by thresholding the combined adjacency matrix $\bm{G}^\star =
\mathbb{I}[W > \tau]$.

\begin{remark}
    \textit{In \eqref{eq:schur}, the low-rank component $L$ is subtracted from
    the sparse component $S$ to yield the precision matrix $\tilde{\Theta}_O$,
    consistent with the decomposition of latent-variable graphical models.
    However, in \eqref{eq:weights}, we instead combine $S$ and $L$ additively to
    construct the combined adjacency matrix $W$. This shift reflects a change in
    objective: while the decomposition aims to model conditional independencies,
    super-structure learning seeks to robustly identify potential (either direct
    or confounded) edges. As an analogy to FCI~\cite{zhang2008completeness}, $L$
    can be viewed as capturing strong marginal associations that manifest as
    bidirected edges in the underlying mixed graph. By combining both
    components, our super-structure $\bm{G}^\star$ conservatively includes all
    plausible causal and confounded links, ensuring no true edge is pruned
    prematurely. This strategy aligns with the goal of constructing a reliable
    (yet not necessarily minimal) super-graph of the causal structure.}
\end{remark}

\noindent\textbf{Hypergraph Interpretation.} Because $L$ captures shared
confounding patterns that typically manifest as bidirected edges in ADMGs,
\tool's combined adjacency $W$ implicitly encodes a hypergraph over observed
variables. This provides structural information that
LVGL~\cite{chandrasekaran2012latent} usually discards, and is crucial for
ensuring that no true causal relation is eliminated during thresholding.

The following proposition states that the super-structure learned by \tool is a
valid super structure, i.e., it contains all the edges in the true causal graph,
and thus can be used to guide the differentiable optimization process of
existing differentiable causal discovery methods
(\F~\ref{fig:alvgl_reorganized}~(c)).
\begin{proposition}
    \label{prop:super-structure}
    Let $\bm{G}^\star$ be the super-structure learned by \tool, and let
    $\bm{G}^*$ be the true causal graph. Under assumptions A1--A2,
    $\bm{G}^\star$ is a super structure of $\bm{G}^*$, i.e., $\bm{G}^\star
    \supseteq \bm{G}^*$, with or without latent confounders.
\end{proposition}

\begin{proof}
We prove that the super-structure $\bm{G}^\star$ learned by \tool contains all
edges of the true causal graph $\bm{G}^*$, regardless of the presence of latent
confounders. For both DAGs and MAGs, under the linear Gaussian assumption (A1)
and faithfulness (A2), any edge $\{i,j\} \in \bm{G}^*$ implies that $X_i
\not\Perp X_j \mid \bm{X}_O \setminus \{X_i, X_j\}$, meaning ${\Theta}_{O,ij}
\neq 0$, where ${\Theta}_O = S$. Under the identifiability
conditions~\cite{chandrasekaran2012latent}, the estimator \eqref{eq:objective}
can recovers the support of $S$, so with high probability (\textit{w.h.p.}),
$|S_{ij}| + |L_{ij}| > \tau$ for all true edges, provided $\tau$ is chosen
appropriately (e.g., $\tau = 0$ for a conservative estimate). Since $W = |S| +
|L|$ is used to construct $\bm{G}^\star = \mathbb{I}[W > \tau]$, this guarantees
$\bm{G}^\star \supseteq \bm{G}^*$ \textit{w.h.p.} Thus, $\bm{G}^\star$ is a
valid super-structure of $\bm{G}^*$.
\end{proof}

In the absence of latent confounders, the super-structure learned by \tool can
be characterized as a moralized DAG of the true causal graph, with the following
proposition stating the property.
\begin{proposition}
    \label{prop:moralized-dag}
    Let $\bm{G}^\star$ be the super-structure learned by \tool, and let
    $\bm{G}^*$ be the true causal graph. If there are no latent confounders and
    the assumptions of \Prop~\ref{prop:super-structure} hold, then
    $\bm{G}^\star$ is a moralized DAG of $\bm{G}^*$.
\end{proposition}

\begin{proofsketch}
When there are no latent confounders, the low-rank component $L$ is zero and
$\Theta_O$ is equal to $\tilde{\Theta}_O$. Then, the above proposition is a
direct consequence of Theorem 1 in~\cite{ng2021reliable}.
\end{proofsketch}

\begin{corollary}
    With \Prop~\ref{prop:moralized-dag} and \Lem~A.3
    of~\cite{nazaret2024stable}, if $\bm{G}^*$ has maximal in-degree of $k$ and
    let $d$ be the number of observed variables, then $\bm{G}^\star$ has at most
    $\mathcal{O}(dk^2)$ edges, which is usually much smaller than the
    $\mathcal{O}(d^2)$ edges in a full adjacency matrix.
    \label{cor:super-structure-size}
\end{corollary}

\subsection{Computational Complexity}
\label{subsec:complexity}

For each ADMM iteration, the dominant costs are a matrix inversion
($\mathcal{O}(d^3)$) and an SVD with rank truncation. This is $\mathcal{O}(d^3)$
in the worst case, but can be reduced to $\mathcal{O}(d^2 r^\star)$ using
economy decomposition, where $d$ is the number of observed variables and
$r^\star$ is the effective rank of the low-rank component $L$. The method
usually takes a small fraction of the entire runtime of differentiable causal
discovery methods, while providing a significant boost in optimization
efficiency, as will be shown in \S~\ref{sec:eval}.

\subsection{Integration in Differentiable Optimization}
\label{subsec:edge-masking}

\tool outputs a binary super-structure mask $\bm{G}^\star \in \{0,1\}^{d\times
d}$ encoding the set of edges that are \emph{allowed} to be nonzero. Let
$\mathcal{Z}=\{(i,j)\mid \bm{G}^\star_{ij}=0\}$ denote the forbidden edges. A
differentiable causal discovery method solves
\begin{equation}
    \min_{W\in\mathbb{R}^{d\times d}} \;\mathcal{J}(W)
    \quad\text{s.t.}\quad h(W)=0,
\end{equation}
where $h(W)$ enforces acyclicity or other structural constraints. Incorporating
the \tool super-structure introduces additional linear equality constraints
\begin{equation}
    W_{ij}=0\quad\text{for all } (i,j)\in\mathcal{Z},
\end{equation}
which restrict the feasible set to the linear subspace
\begin{equation}
\mathcal{W}_{\bm{G}^\star} := \{\,W\in\mathbb{R}^{d\times d} : W\circ(1-\bm{G}^\star)=0\,\},    
\end{equation}
reducing the intrinsic optimization dimension from $d^2$ parameters to
$\|\bm{G}^\star\|_0$. By \Prop~\ref{prop:super-structure}, \tool recovers a
valid super-structure \textit{w.h.p.}, meaning every true edge $(i,j)$ satisfies
$\bm{G}^\star_{ij}=1$. Thus, if $W^\star$ denotes the true weighted adjacency
matrix, $W^\star\in\mathcal{W}_M$, and the following solution-preservation
result holds.

\begin{proposition}[Solution Preservation]
If $\bm{G}^\star$ is a valid super-structure for the true solution $W^\star$
(i.e., $\bm{G}^\star_{ij}=0\implies W^\star_{ij}=0$), then any global minimizer
of the full problem remains a global minimizer of the restricted problem
\begin{equation}
\min_{W\in\mathcal{W}_{\bm{G}^\star}}\mathcal{J}(W)
\quad\text{s.t.}\quad h(W)=0.
\end{equation}
\end{proposition} 

This implies that restricting the domain does not eliminate valid solutions.
Instead, it removes an exponentially large portion of structurally impossible
configurations. In convex analogues, projecting onto such a constraint subspace
both preserves the optimum and can \emph{improve conditioning} by eliminating
noisy gradient directions, yielding faster convergence. While $h(W)$ introduces
non-convexity, empirical results in \S\ref{sec:eval} show that the same benefits
manifest in practice: faster convergence, more stable L--BFGS updates, and large
runtime reductions in high-dimensional regimes.

\paragraph{Implementation via projected gradients.}
During optimization, we maintain $W^{(0)} \in \mathcal{W}_M$ and update only
free parameters. At iteration $k$, if $G^{(k)}=\nabla\mathcal{J}(W^{(k)})$ is
the gradient, the projected update is
\begin{equation}
    G^{(k)}_{\text{proj}} = G^{(k)} \circ M,
\end{equation}
ensuring forbidden edges never change. Similarly, any second-order information
is projected to $\mathcal{W}_{\bm{G}^\star}$ as well. The L--BFGS solver (or any
gradient-based optimizer) then operates entirely within
$\mathcal{W}_{\bm{G}^\star}$. This direct enforcement of hard constraints
differs from soft mask regularization approaches, which merely discourage
unwanted edges. Because \tool achieves consistently high recall, the risk of
excluding true edges is negligible, making hard constraints both safe and far
more computationally effective. This principle is also adopted in recent causal
discovery methods~\cite{ma2024scalable,nazaret2024stable}. 

\section{Evaluation}
\label{sec:eval}
In evaluation, we aim to evaluate the performance of \tool from three
perspectives: \ding{202} universal enhancement to diverse differentiable causal
discovery methods; \ding{203} comparison with other super-structure learning
methods; \ding{204} robustness to diverse settings and misspecified data and
\ding{205} a case study on real-world datasets.

\parh{Data Generation.}~Our experiments involve two settings, including linear
Gaussian SCMs with latent confounders and linear Gaussian SCMs without latent
confounders. For the first setting, we follow the same parametric setup as
ABIC~\cite{bhattacharya2021differentiable} and SPOT~\cite{ma2024scalable} and
generate ADMGs with 15--25 nodes and 1000 samples. For the second setting, we
follow the same parametric setup as NOTEARS~\cite{zheng2018dags} and generate
DAGs with 100--200 nodes and 1000 samples. For each setting, we generate
datasets with Erdos-Renyi (ER) graphs with a fixed degree of 1 in main
comparison. We also explore the effectiveness of \tool on larger graphs,
different graph structures, including scale-free (SF) graphs and bipartite (BP)
graphs, degrees and sample sizes in \S~\ref{subsec:rq3}.

\parh{Baselines.}~We implement \tool on four representative differentiable
causal discovery methods: NOTEARS~\cite{zheng2018dags} (2018),
GOLEM~\cite{ng2020role} (2020), DAGMA~\cite{bello2022dagma} (2022), and
ABIC~\cite{bhattacharya2021differentiable} (2021). The first three methods are
designed for linear Gaussian data without latent confounders, while ABIC is
designed for linear Gaussian data with latent confounders. We denote the
resulting methods as NOTEARS$^\star$, GOLEM$^\star$, DAGMA$^\star$, and
ABIC$^\star$, respectively.

In addition, we implement graphical lasso (GLasso)~\cite{friedman2008sparse},
latent variable graphical lasso (LVGL)~\cite{chandrasekaran2012latent}, and
skeleton learning~\cite{ma2024scalable} as alternative super structure learning
methods. These super structures replicate the functionality of
SDCD~\cite{nazaret2024stable} (2024), \tool's ablation without the optimizations
described in \S~\ref{sec:dlvgl}, and SPOT~\cite{ma2024scalable} (2024),
respectively. We denote these methods as GLasso, LVGL, and SL, and use NOTEARS
as the base differentiable causal discovery method.

All baselines use the default hyperparameters suggested in their documentation.
For \tool, we set the hyperparameters $\lambda_s = 0.05$, $\lambda_l = 0.05$,
$\tau = 1\times10^{-6}$, and $T = 500$ for all experiments.

\subsection{End-to-end Comparison}
\label{subsec:rq1}

\begin{figure}[t]
    \centering
    \includegraphics[width=0.75\linewidth]{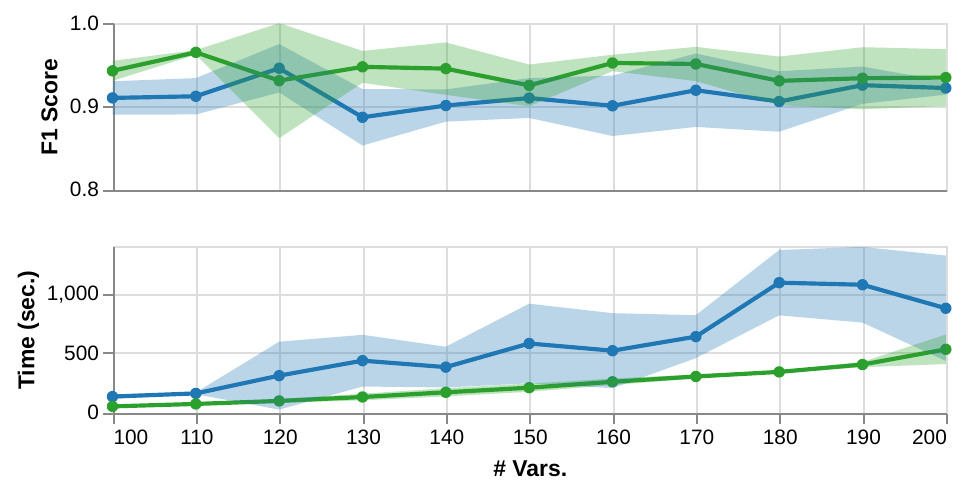}
    \caption{ \textcolor{vegaBlue}{NOTEARS} vs.
	\textcolor{vegaGreen}{NOTEARS$^\star$}.}
    \label{fig:rq1:notears}
\end{figure}

\begin{figure}
    \centering
    \includegraphics[width=0.75\linewidth]{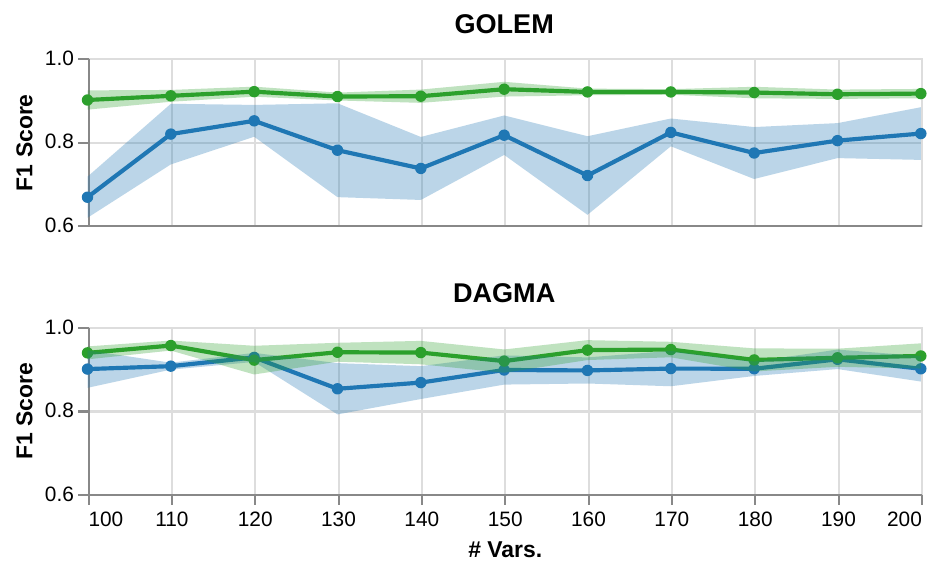}
    \caption{ \textcolor{vegaBlue}{GOLEM} vs.
	\textcolor{vegaGreen}{GOLEM$^\star$} (upper) and \textcolor{vegaBlue}{DAGMA}
	vs. \textcolor{vegaGreen}{DAGMA$^\star$.}  (lower)}
    \label{fig:rq1:golem-dagma}
\end{figure}

\parh{NOTEARS.}~\F~\ref{fig:rq1:notears} shows that {NOTEARS$^\star$}
consistently outperforms {NOTEARS} on synthetic datasets, achieving an average
F1 improvement of 3.3\%. Runtime is reduced by 52.9\%, and importantly, variance
in runtime is also substantially lower (29.30 vs. 233.09). This reflects the
stabilizing effect of \tool, which constrains the search space via learned
super-structures. Furthermore, NOTEARS$^\star$ scales more gracefully with graph
size, which validates \Cor~\ref{cor:super-structure-size}, as its search space
does not grow quadratically like NOTEARS.

\parh{GOLEM and DAGMA.}~\F~\ref{fig:rq1:golem-dagma} compares GOLEM and DAGMA
with their \textcolor{vegaGreen}{$^\star$} counterparts. We report only F1
scores due to fixed iteration settings ($1{\times}10^5$ for GOLEM and
$1.8{\times}10^6$ for DAGMA), which yield deterministic runtimes. GOLEM$^\star$
and DAGMA$^\star$ consistently outperform their baselines, with F1 improvements
of 17.5\% and 4.0\%, respectively. These gains are attributed to more effective
convergence enabled by \tool, especially in high-dimensional regimes where
baseline methods tend to under-optimize.

\begin{figure}[t]
    \centering
    \includegraphics[width=0.75\linewidth]{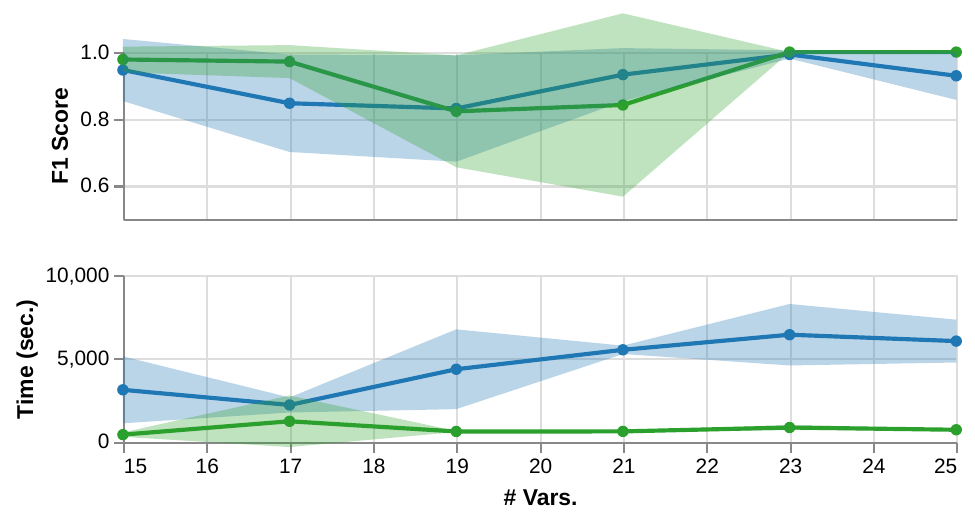}
    \caption{ \textcolor{vegaBlue}{ABIC} vs.
	\textcolor{vegaGreen}{ABIC$^\star$}.}
    \label{fig:rq1:abic}
\end{figure}

\parh{ABIC.}~\F~\ref{fig:rq1:abic} shows results for ABIC vs. ABIC$^\star$ on
smaller graphs (15--25 nodes) under linear Gaussian SCMs with latent
confounders. We follow the default strategy of selecting the best BIC-scoring
graph across five runs. On average, ABIC$^\star$ improves F1 by 3.1\% and
reduces runtime by 77.4\%. The runtime gain is particularly significant due to
the simplification introduced by super-structure learning in this
lower-dimensional setting. One outlier is observed at 21 nodes, where
ABIC$^\star$ underperforms, likely due to instability in optimizing compound
graphs under latent confounding.

\subsection{Super-Structure Learning Effectiveness}
\label{subsec:rq2}

\begin{figure}
    \centering
    \includegraphics[width=0.75\linewidth]{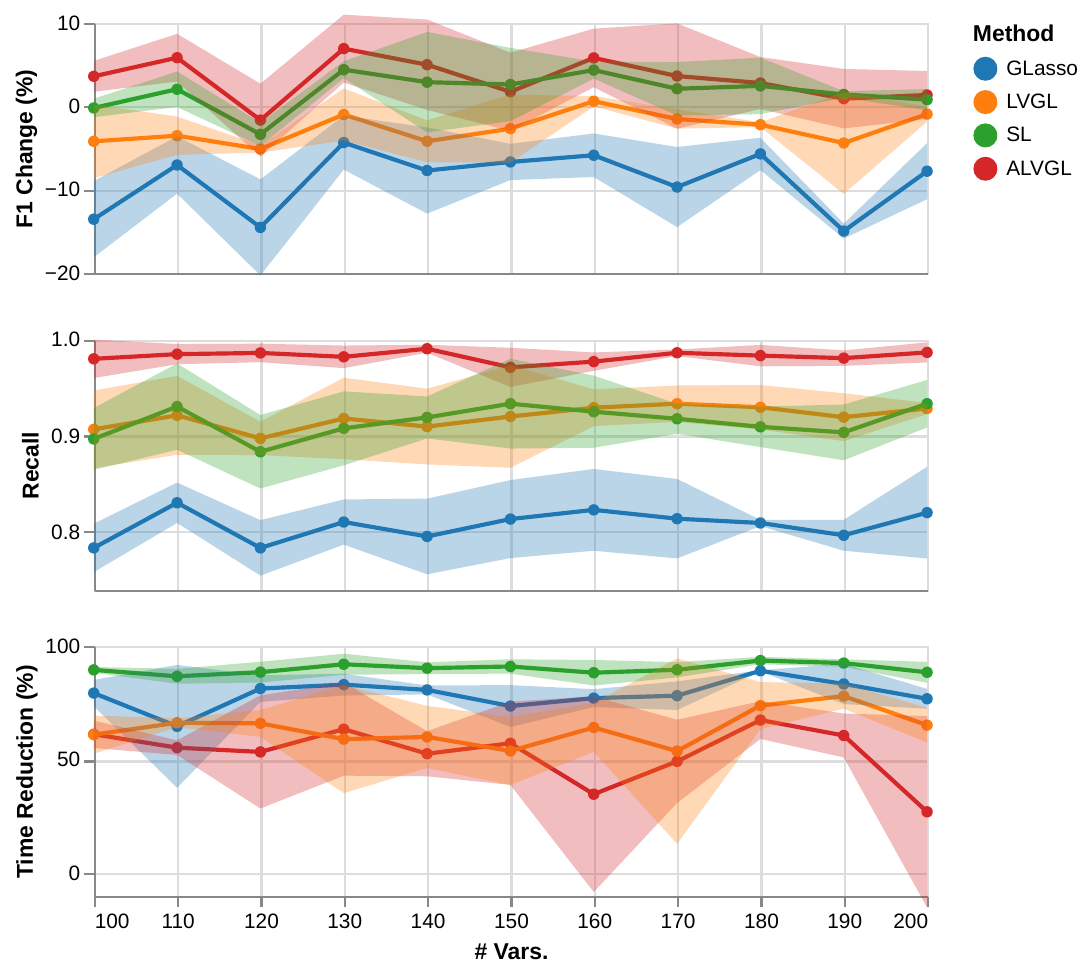}
    \caption{Comparison of different super-structure learning methods on
	NOTEARS.}
    \label{fig:rq2}
\end{figure}

To assess the unique advantages of \tool in differentiable causal discovery, we
compare it against three alternative super-structure learning approaches applied
to NOTEARS:
\begin{itemize}
\item \textbf{GLasso}: Precision matrix estimation without low-rank modeling,
conceptually similar to SDCD~\cite{nazaret2024stable};
\item \textbf{LVGL}: ADMM-based latent-variable graphical
Lasso~\cite{ma2013alternating}, omitting our inexact ADMM update and
super-structure learning components;
\item \textbf{SL}: Skeleton-based initialization using the PC algorithm, similar
to SPOT~\cite{ma2024scalable}.
\end{itemize}
\F~\ref{fig:rq2} reports relative changes in F1 score and runtime, along with
the average recall of the learned super-structures—crucial for enabling
effective downstream optimization. \tool achieves the highest F1 improvement
(+3.3\%), followed by SL (+1.8\%), while GLasso and LVGL underperform, primarily
due to low super-structure recall. For example, GLasso captures only $0.81 \pm
0.016$ of true edges, causing irreversible omission during NOTEARS optimization.
In contrast, \tool\ achieves a recall of $0.98 \pm 0.005$, ensuring most true
edges are retained and refined. All methods reduce NOTEARS runtime, with \tool
achieving a 52.9\% speedup despite its larger super-structure. This modest
trade-off reflects its broader search space, yet still confirms its efficiency
in guiding the optimization.

\subsection{Robustness}
\label{subsec:rq3}

We evaluate the robustness of \tool\ under varying configurations: graph types,
node degrees, sample sizes, and noise distributions. These experiments 
demonstrate \tool's ability to maintain performance advantages across diverse
scenarios while revealing important insights about when and why the method
provides greater benefits.

\begin{table}[h]
\centering
\caption{F1 score and runtime (sec.) comparison of NOTEARS with and without
\tool. The best results are in bold. Varying graph types.}
\label{tab:rq3:graph-type}
\footnotesize
\begin{tabular}{lccc}
\toprule
Graph Type & ER & SF & BP  \\
\midrule
F1 w/o \tool& 0.928 $\pm$ 0.023 & 0.979 $\pm$ 0.011  & 0.882 $\pm$ 0.017 \\
F1 w/ \tool& \textbf{0.948 $\pm$ 0.009} & \textbf{0.980 $\pm$ 0.009}  & \textbf{0.988 $\pm$ 0.011} \\
Time w/o \tool& 349.1 $\pm$ 49.5 & 145.7 $\pm$ 45.0 & 185.6 $\pm$ 26.4 \\
Time w/ \tool& \textbf{189.5 $\pm$ 12.7} & \textbf{193.9 $\pm$ 19.2} & \textbf{173.9 $\pm$ 6.6} \\
\bottomrule
\end{tabular}
\end{table}

\parh{Graph Types.}~While our main analysis focuses on ER graphs, real-world
graphs often follow different topologies, such as SF or BP structures. As shown
in \T~\ref{tab:rq3:graph-type}, \tool consistently improves F1 scores over
vanilla differentiable causal discovery algorithms across all types, with
notable gains on BP graphs (12.0\% improvement) and demonstrates its ability to
capture core causal dependencies even in non-standard structures. The
particularly strong performance on BP graphs can be attributed to their more
structured sparsity patterns, which align well with \tool's sparse+low-rank
decomposition: the bipartite structure naturally induces block-sparse patterns
in the precision matrix that are efficiently captured by the sparse component
$S$, leading to more accurate super-structure learning.

\begin{table}[h]
\centering
\caption{F1 score and runtime (sec.) comparison of NOTEARS with and without \tool. The best results are in bold. Varying node degree.}
\label{tab:rq3:graph-density}
\footnotesize
\begin{tabular}{lccccccc}
\toprule
Degree & 1.00 & 1.25 & 1.50 & 1.75 & 2.00 & 2.25 & 2.50 \\
\midrule
F1 w/o \tool & 0.93 $\pm$ 0.02 & \textbf{0.95 $\pm$ 0.02} & \textbf{0.95 $\pm$ 0.02} & 0.93 $\pm$ 0.02 & \textbf{0.96 $\pm$ 0.01} & \textbf{0.96 $\pm$ 0.01} & \textbf{0.96 $\pm$ 0.01} \\
F1 w/ \tool  & \textbf{0.95 $\pm$ 0.01} & \textbf{0.95 $\pm$ 0.01} & \textbf{0.95 $\pm$ 0.01} & \textbf{0.95 $\pm$ 0.01} & 0.95 $\pm$ 0.01 & \textbf{0.96 $\pm$ 0.01} & 0.95 $\pm$ 0.01 \\
Time w/o \tool & 350 $\pm$ 51 & 758 $\pm$ 128 & 1122 $\pm$ 140 & 1188 $\pm$ 183 & 1731 $\pm$ 52 & 1709 $\pm$ 44 & 1670 $\pm$ 282 \\
Time w/ \tool  & \textbf{191 $\pm$ 13} & \textbf{264 $\pm$ 58} & \textbf{472 $\pm$ 242} & \textbf{374 $\pm$ 121} & \textbf{707 $\pm$ 176} & \textbf{838 $\pm$ 360} & \textbf{814 $\pm$ 151} \\
\bottomrule
\end{tabular}
\end{table}

\parh{Node Degree.}~\T~\ref{tab:rq3:graph-density} reveals an interesting
pattern: \tool's F1 improvement over vanilla algorithms is most pronounced at
lower degrees (1.0--1.75), while performance becomes comparable at higher
degrees (2.0--2.5). This observation reflects the inherent trade-off in
super-structure learning. At lower degrees, the true causal graph is sparser,
and \tool's ability to prune the search space from $\mathcal{O}(d^2)$ to
$\mathcal{O}(dk^2)$ (as stated in \Cor~\ref{cor:super-structure-size}) provides
substantial benefits by eliminating spurious edges that would otherwise mislead
vanilla algorithms. However, at higher degrees, two competing effects emerge:
\ding{192} the moralized graph becomes denser, reducing the relative advantage
of super-structure pruning, and \ding{193} vanilla algorithms themselves perform
better on denser graphs as the increased signal strength makes the optimization
landscape smoother. Despite this, \tool maintains roughly 2$\times$-3$\times$
(1.8$\times$-3.2$\times$) runtime improvements across all degree settings,
confirming its computational efficiency even when accuracy gains diminish. The
runtime benefits warrant further explanation. Even though higher-degree graphs
reduce the sparsity advantage, \tool's super-structure initialization provides a
better starting point for gradient-based optimization. By constraining the
search space (as described in \S~\ref{subsec:edge-masking}), \tool prevents the
optimizer from exploring large regions of infeasible solutions, leading to
faster convergence regardless of the final graph density.

\begin{table}[h]
\centering
\caption{F1 score and runtime (sec.) comparison of NOTEARS with and without \tool. The best results are in bold. Varying sample size.}
\label{tab:rq3:sample-size}
\footnotesize
\begin{tabular}{lcccc}
\toprule
$n$ & 500 & 1000 & 5000 & 10000 \\
\midrule
F1 w/o \tool & 0.925 $\pm$ 0.027 & 0.928 $\pm$ 0.023 & 0.925 $\pm$ 0.025 & 0.929 $\pm$ 0.019 \\
F1 w/ \tool  & \textbf{0.949 $\pm$ 0.003} & \textbf{0.948 $\pm$ 0.009} & \textbf{0.965 $\pm$ 0.023} & \textbf{0.954 $\pm$ 0.035} \\
Time w/o \tool & 646.1 $\pm$ 372.9 & 377.5 $\pm$ 71.7 & 619.0 $\pm$ 268.7 & 941.8 $\pm$ 279.3 \\
Time w/ \tool  & \textbf{247.0 $\pm$ 107.1} & \textbf{195.3 $\pm$ 16.6} & \textbf{572.4 $\pm$ 193.8} & \textbf{933.7 $\pm$ 45.7} \\
\bottomrule
\end{tabular}
\end{table}

\parh{Sample Size.}~\T~\ref{tab:rq3:sample-size} shows that F1 improvements
remain modest but consistent across all sample sizes, ranging from an absolute
increase of 0.024 at $n=500$ to 0.025 at $n=10000$ (corresponding to roughly
2.6-2.7\% relative improvement over an already high baseline). In contrast,
runtime benefits diminish substantially, from a 61.8\% reduction at $n=500$ to
only 0.9\% at $n=10000$. This pattern reflects the changing computational
bottleneck as sample size increases. At small sample sizes, vanilla algorithms
face two challenges: \ding{192} high variance in covariance estimation leads to
noisy gradients and slower convergence, and \ding{193} the search space contains
many spurious edges that must be explored and rejected. \tool addresses both:
its low-rank component $L$ stabilizes the precision matrix estimate (as
motivated in \S~\ref{subsec:glasso}), and its super-structure pruning eliminates
spurious edges a priori. The combination yields substantial speedups (62\%
faster). However, as $n$ increases, the computational dynamics shift. With more
samples, the empirical covariance $\hat{\Sigma}$ becomes well-conditioned,
making gradient estimation reliable even in the full $\mathcal{O}(d^2)$ search
space. vanilla algorithms themselves can now converge efficiently without
super-structure guidance where the optimization is no longer bottlenecked by
search space size but by the fundamental cost of processing large datasets.
Specifically, each gradient computation scales as $\mathcal{O}(nd^2+d^3)$, and
at $n=10000$, this sample size-dependent gradient computation cost dominates any
savings from pruning edges. Meanwhile, \tool still incurs the overhead of
super-structure learning (covariance computation and ADMM), which becomes
relatively more expensive as vanilla algorithms' own runtime stabilizes. This
explains the vanishing runtime advantage: at $n=10000$, both methods spend most
time computing gradients over large data matrices rather than exploring the
graph structure. The $\sim$1\% speedup reflects only the marginal benefit of
reducing the per-iteration edge count, which is overshadowed by the fixed cost
of gradient computation. Having said that, given that many real-world
applications operate in small-sample regimes ($\sim$100--1000 samples) due to
data collection constraints, \tool's significant runtime advantages in these
settings remain highly relevant and valuable.

\begin{table}[h]
\centering
\caption{F1 score and runtime (sec.) comparison of NOTEARS with and without \tool. The best results are in bold. Varying noise type.}
\label{tab:rq3:noise-type}
\footnotesize
\begin{tabular}{lcccc}
\toprule
Noise Type & Gaussian & Exponential & Gumbel & Uniform \\
\midrule
F1 w/o \tool & 0.928 $\pm$ 0.023 & \textbf{0.925 $\pm$ 0.016} & \textbf{0.914 $\pm$ 0.023} & 0.840 $\pm$ 0.017 \\
F1 w/ \tool  & \textbf{0.948 $\pm$ 0.009} & 0.914 $\pm$ 0.018 & 0.904 $\pm$ 0.043 & \textbf{0.902 $\pm$ 0.032} \\
Time w/o \tool & 351.5 $\pm$ 41.5 & 455.2 $\pm$ 151.4 & 526.3 $\pm$ 127.5 & 298.6 $\pm$ 177.9 \\
Time w/ \tool  & \textbf{189.4 $\pm$ 16.0} & \textbf{257.4 $\pm$ 53.1} & \textbf{210.0 $\pm$ 25.1} & \textbf{180.2 $\pm$ 17.2} \\
\bottomrule
\end{tabular}
\end{table}

\parh{Non-Gaussian Noise.}~Although \tool is grounded in linear Gaussian
assumptions, we test its robustness on data with exponential, Gumbel, and
uniform noise (see \T~\ref{tab:rq3:noise-type}). Results show graceful
degradation: F1 scores remain competitive under exponential noise (with slight
degradation) and improve substantially under uniform noise (from 0.84 to 0.90).
The varying performance can be explained by the degree of distributional
mismatch. Exponential and Gumbel distributions are asymmetric and heavy-tailed,
violating Gaussianity more severely and potentially introducing bias in the
precision matrix estimate. In contrast, uniform noise, while non-Gaussian, is
symmetric and bounded, leading to less severe misspecification. Under uniform
noise, \tool's low-rank component effectively captures deviations from
Gaussianity as latent structure, inadvertently improving robustness.
Importantly, \tool retains substantial runtime benefits across all noise types
(about 40-60\% reduction), which suggests strong practical robustness to
distributional misspecification. We presume this is because the super-structure
learning step primarily relies on correlation patterns rather than exact
distributional assumptions, allowing it to guide optimization effectively even
when the Gaussian assumption is violated. NOTEARS with or without \tool\ fails
under logistic and Poisson noise due to optimization instability and are
excluded from this analysis; exploring discrete and non-linear extensions
remains future work.

\parh{Summary.}~These robustness experiments reveal that \tool provides the
greatest accuracy improvements when: \ding{192} graphs are sparse (low degree),
\ding{193} sample sizes are limited, and \ding{194} noise distributions are
Gaussian or moderately non-Gaussian but symmetric. However, computational
benefits persist across all settings, making \tool a universally beneficial
enhancement for differentiable causal discovery pipelines, particularly in
high-dimensional regimes where optimization efficiency is critical.

\subsection{Real-world Data}
\label{subsec:rq4}

\begin{table}[h]
\centering
\caption{Precision, recall and F1 score on Sachs dataset. The best results are in bold.}
\label{tab: rq4:sachs}
\footnotesize
\begin{tabular}{lcccc}
\toprule
Methods & NOTEARS$^\star$ & NOTEARS w/ GLasso & NOTEARS w/ LVGL & NOTEARS \\
\midrule
Precision & 0.67 & 0.63 & \textbf{0.71} & 0.57 \\
Recall & \textbf{0.32} & 0.26 & 0.26 & 0.21 \\
F1-Score & \textbf{0.43} & 0.37 & 0.38 & 0.31 \\
\bottomrule
\end{tabular}
\end{table}

We evaluate \tool on the Sachs dataset~\cite{sachs2005causal}, which contains
7466 cells ($n = 7466$) and flow cytometry measurements of 11 ($d = 11$)
phosphorylated proteins and phospholipids. The F1 score is computed against the
consensus CPDAG ground truth. As shown in \T~\ref{tab: rq4:sachs},
NOTEARS$^\star$ achieves the highest F1 score of 0.43, significantly
outperforming the baseline NOTEARS score of 0.31.

In this setting ($n \gg d$), while covariance estimation is stable, the primary
challenge stems from complex biological dependencies that are not fully captured
by the strict linear Gaussian assumptions of the model. This misspecification
creates a difficult optimization landscape, leading baseline NOTEARS to miss a
large proportion of true causal edges, resulting in a poor recall of 0.21. The
key advantage of \tool lies in providing a reliable estimation of the
super-structure to guide the search. By identifying a comprehensive set of
plausible edges that capture both direct signals and associations potentially
obscured by confounding, it enables NOTEARS to focus its optimization on the
variable pairs that matter most in the underlying graph. This targeted search
leads to a substantial improvement in recall to 0.32 and the best overall F1
score.

Comparing the alternatives highlights the effectiveness of \tool's specific
design. NOTEARS w/ GLasso ignores potential latent structures, yielding only
partial improvement. NOTEARS w/ LVGL employs a standard decomposition that
appears too conservative for this task; while it achieves high precision (0.71),
its low recall (0.26) suggests it prematurely prunes true edges from the search
space. NOTEARS$^\star$, by explicitly combining sparse and low-rank evidence
into the super-structure, strikes the best balance, maximizing recall among
super-structure methods and enabling the downstream optimizer to reconstruct the
most accurate graph.

Unlike our synthetic benchmarks, Sachs data does not strictly satisfy linear
Gaussian assumptions. The observed improvement here is complementary to findings
in \S~\ref{subsec:rq3}, demonstrating \tool's practical robustness. Even in
low-dimensional, large-sample real-world scenarios, its approach to reliably
constraining the search space provides a critical advantage over unstructured
optimization, leading to more accurate causal discovery. \section{Related Work}

\parh{Differentiable Causal Discovery.}~Differentiable methods have recently
emerged as a promising approach for improving the accuracy and efficiency of
causal discovery. These methods reformulate the problem as a continuous
optimization task, enabling the use of gradient-based algorithms to learn causal
structures directly from data. Notable examples include
NOTEARS~\cite{zheng2018dags}, GOLEM~\cite{ng2020role},
DAGMA~\cite{bello2022dagma}, and ABIC~\cite{bhattacharya2021differentiable}.
Subsequent work has focused on enhancing performance through various means:
improving the stability and computational efficiency of acyclicity
constraints~\cite{yu2019dag, wei2020dags, zheng2020learning, bello2022dagma,
zhang2025analytic}; incorporating domain knowledge and interventional
data~\cite{ban2024differentiable, brouillard2020differentiable,
faria2022differentiable, li2023causal}; extending expressiveness to handle more
complex functional forms~\cite{lorch2021dibs, ashman2022causal,
geffner2022deep}; and increasing robustness to model
misspecification~\cite{zhangboosting, yi2025robustness}. In this paper, we
contribute a novel yet complementary enhancement by introducing a reliable
super-structure learning step into the optimization pipeline, which
significantly improves optimization efficiency.

\parh{Structure Learning.}~Structure learning is a central task in probabilistic
graphical models (PGMs), where the goal is to uncover the underlying conditional
independence structure among variables. Classical approaches typically focus on
learning structures such as Markov blankets~\cite{ling2019bamb}, conditional
independence graphs~\cite{meinshausen2006high}, or full network
structures~\cite{peters2017elements}, using score-based, constraint-based, or
hybrid algorithms. In high-dimensional settings, sparsity assumptions and
regularization techniques such as graphical Lasso~\cite{friedman2008sparse} have
been widely adopted to ensure statistical and computational tractability. While
much of this literature focuses on identifying precise conditional dependencies,
we instead aim to learn a coarser but reliable super-structure that guides
downstream differentiable optimization. Despite its simplicity, our approach
outperforms existing methods targeting exact structures on downstream tasks.

\section{Conclusion}

In this paper, we introduced \tool, a novel and general enhancement to the
differentiable causal discovery pipeline. By learning a super-structure that
captures the core causal dependencies, \tool significantly reduces the search
space and improves optimization efficiency. We demonstrated the versatility of
\tool across various structural causal models, including Gaussian and
non-Gaussian settings, with and without unmeasured confounders. Extensive
experiments on synthetic and real-world datasets showed that \tool not only
achieves state-of-the-art accuracy but also significantly improves optimization
efficiency, making it a reliable and efficient solution for differentiable
causal discovery.

\bibliographystyle{ACM-Reference-Format}
\bibliography{main}

@inproceedings{bhattacharya2021differentiable,
  title={Differentiable causal discovery under unmeasured confounding},
  author={Bhattacharya, Rohit and Nagarajan, Tushar and Malinsky, Daniel and Shpitser, Ilya},
  booktitle={International Conference on Artificial Intelligence and Statistics},
  pages={2314--2322},
  year={2021},
  organization={PMLR}
}

@article{zheng2018dags,
  title={Dags with no tears: Continuous optimization for structure learning},
  author={Zheng, Xun and Aragam, Bryon and Ravikumar, Pradeep K and Xing, Eric P},
  journal={Advances in neural information processing systems},
  volume={31},
  year={2018}
}

@article{ng2021reliable,
  title={Reliable causal discovery with improved exact search and weaker assumptions},
  author={Ng, Ignavier and Zheng, Yujia and Zhang, Jiji and Zhang, Kun},
  journal={Advances in Neural Information Processing Systems},
  volume={34},
  pages={20308--20320},
  year={2021}
}

@inproceedings{ma2024scalable,
  title={Scalable Differentiable Causal Discovery in the Presence of Latent Confounders with Skeleton Posterior},
  author={Ma, Pingchuan and Ding, Rui and Fu, Qiang and Zhang, Jiaru and Wang, Shuai and Han, Shi and Zhang, Dongmei},
  booktitle={Proceedings of the 30th ACM SIGKDD Conference on Knowledge Discovery and Data Mining},
  pages={2141--2152},
  year={2024}
}

@article{chandrasekaran2012latent,
  title={Latent variable graphical model selection via convex optimization},
  author={Chandrasekaran, Venkat and Parrilo, Pablo A and Willsky, Alan S},
  journal={The Annals of Statistics},
  pages={1935--1967},
  year={2012},
  publisher={JSTOR}
}

@article{ban2024differentiable,
  title={Differentiable structure learning with partial orders},
  author={Ban, Taiyu and Chen, Lyuzhou and Wang, Xiangyu and Wang, Xin and Lyu, Derui and Chen, Huanhuan},
  journal={Advances in Neural Information Processing Systems},
  volume={37},
  pages={117426--117455},
  year={2024}
}

@article{brouillard2020differentiable,
  title={Differentiable causal discovery from interventional data},
  author={Brouillard, Philippe and Lachapelle, S{\'e}bastien and Lacoste, Alexandre and Lacoste-Julien, Simon and Drouin, Alexandre},
  journal={Advances in Neural Information Processing Systems},
  volume={33},
  pages={21865--21877},
  year={2020}
}

@inproceedings{faria2022differentiable,
  title={Differentiable causal discovery under latent interventions},
  author={Faria, Gon{\c{c}}alo Rui Alves and Martins, Andre and Figueiredo, M{\'a}rio AT},
  booktitle={Conference on Causal Learning and Reasoning},
  pages={253--274},
  year={2022},
  organization={PMLR}
}

@article{li2023causal,
  title={Causal discovery from observational and interventional data across multiple environments},
  author={Li, Adam and Jaber, Amin and Bareinboim, Elias},
  journal={Advances in Neural Information Processing Systems},
  volume={36},
  pages={16942--16956},
  year={2023}
}

@inproceedings{zhangboosting,
  title={Boosting Causal Discovery via Adaptive Sample Reweighting},
  author={Zhang, An and Liu, Fangfu and Ma, Wenchang and Cai, Zhibo and Wang, Xiang and Chua, Tat-Seng},
  booktitle={The Eleventh International Conference on Learning Representations},
  year={2023},
}

@inproceedings{yi2025robustness,
  title={The Robustness of Differentiable Causal Discovery in Misspecified Scenarios},
  author={Yi, Huiyang and He, Yanyan and Chen, Duxin and Kang, Mingyu and Wang, He and Yu, Wenwu},
  booktitle={The Thirteenth International Conference on Learning Representations},
  year={2025}
}

@inproceedings{yu2019dag,
  title={DAG-GNN: DAG structure learning with graph neural networks},
  author={Yu, Yue and Chen, Jie and Gao, Tian and Yu, Mo},
  booktitle={International conference on machine learning},
  pages={7154--7163},
  year={2019},
  organization={PMLR}
}

@article{wei2020dags,
  title={DAGs with No Fears: A closer look at continuous optimization for learning Bayesian networks},
  author={Wei, Dennis and Gao, Tian and Yu, Yue},
  journal={Advances in Neural Information Processing Systems},
  volume={33},
  pages={3895--3906},
  year={2020}
}

@inproceedings{zheng2020learning,
  title={Learning sparse nonparametric dags},
  author={Zheng, Xun and Dan, Chen and Aragam, Bryon and Ravikumar, Pradeep and Xing, Eric},
  booktitle={International Conference on Artificial Intelligence and Statistics},
  pages={3414--3425},
  year={2020},
  organization={Pmlr}
}

@article{bello2022dagma,
  title={Dagma: Learning dags via m-matrices and a log-determinant acyclicity characterization},
  author={Bello, Kevin and Aragam, Bryon and Ravikumar, Pradeep},
  journal={Advances in Neural Information Processing Systems},
  volume={35},
  pages={8226--8239},
  year={2022}
}

@article{zhang2025analytic,
  title={Analytic DAG Constraints for Differentiable DAG Learning},
  author={Zhang, Zhen and Ng, Ignavier and Gong, Dong and Liu, Yuhang and Gong, Mingming and Huang, Biwei and Zhang, Kun and Hengel, Anton van den and Shi, Javen Qinfeng},
  journal={arXiv preprint arXiv:2503.19218},
  year={2025}
}

@article{richardson2002ancestral,
  title={Ancestral graph Markov models},
  author={Richardson, Thomas and Spirtes, Peter},
  journal={The Annals of Statistics},
  volume={30},
  number={4},
  pages={962--1030},
  year={2002},
  publisher={Institute of Mathematical Statistics}
}

@article{friedman2008sparse,
  title={Sparse inverse covariance estimation with the graphical lasso},
  author={Friedman, Jerome and Hastie, Trevor and Tibshirani, Robert},
  journal={Biostatistics},
  volume={9},
  number={3},
  pages={432--441},
  year={2008},
  publisher={Oxford University Press}
}

@article{loh2014high,
  title={High-dimensional learning of linear causal networks via inverse covariance estimation},
  author={Loh, Po-Ling and B{\"u}hlmann, Peter},
  journal={Journal of Machine Learning Research},
  volume={15},
  number={140},
  pages={3065--3105},
  year={2014}
}

@article{zhang2022truncated,
  title={Truncated matrix power iteration for differentiable DAG learning},
  author={Zhang, Zhen and Ng, Ignavier and Gong, Dong and Liu, Yuhang and Abbasnejad, Ehsan and Gong, Mingming and Zhang, Kun and Shi, Javen Qinfeng},
  journal={Advances in Neural Information Processing Systems},
  volume={35},
  pages={18390--18402},
  year={2022}
}

@article{vowels2022d,
  title={D’ya like dags? a survey on structure learning and causal discovery},
  author={Vowels, Matthew J and Camgoz, Necati Cihan and Bowden, Richard},
  journal={ACM Computing Surveys},
  volume={55},
  number={4},
  pages={1--36},
  year={2022},
  publisher={ACM New York, NY}
}

@book{spirtes2000causation,
  title={Causation, prediction, and search},
  author={Spirtes, Peter and Glymour, Clark N and Scheines, Richard and Heckerman, David},
  year={2000},
  publisher={MIT press}
}

@article{zhang2008completeness,
  title={On the completeness of orientation rules for causal discovery in the presence of latent confounders and selection bias},
  author={Zhang, Jiji},
  journal={Artificial Intelligence},
  volume={172},
  number={16-17},
  pages={1873--1896},
  year={2008},
  publisher={Elsevier}
}

@article{tsamardinos2006max,
  title={The max-min hill-climbing Bayesian network structure learning algorithm},
  author={Tsamardinos, Ioannis and Brown, Laura E and Aliferis, Constantin F},
  journal={Machine learning},
  volume={65},
  pages={31--78},
  year={2006},
  publisher={Springer}
}

@article{tsirlis2018scoring,
  title={On scoring maximal ancestral graphs with the max--min hill climbing algorithm},
  author={Tsirlis, Konstantinos and Lagani, Vincenzo and Triantafillou, Sofia and Tsamardinos, Ioannis},
  journal={International Journal of Approximate Reasoning},
  volume={102},
  pages={74--85},
  year={2018},
  publisher={Elsevier}
}

@inproceedings{claassen2022greedy,
  title={Greedy equivalence search in the presence of latent confounders},
  author={Claassen, Tom and Bucur, Ioan G},
  booktitle={Uncertainty in Artificial Intelligence},
  pages={443--452},
  year={2022},
  organization={PMLR}
}

@inproceedings{chen2021integer,
  title={Integer programming for causal structure learning in the presence of latent variables},
  author={Chen, Rui and Dash, Sanjeeb and Gao, Tian},
  booktitle={International Conference on Machine Learning},
  pages={1550--1560},
  year={2021},
  organization={PMLR}
}

@article{rohekar2021iterative,
  title={Iterative causal discovery in the possible presence of latent confounders and selection bias},
  author={Rohekar, Raanan Y and Nisimov, Shami and Gurwicz, Yaniv and Novik, Gal},
  journal={Advances in Neural Information Processing Systems},
  volume={34},
  pages={2454--2465},
  year={2021}
}

@book{peters2017elements,
  title={Elements of causal inference},
  author={Peters, Jonas and Janzing, Dominik and Sch{\"o}lkopf, Bernhard},
  year={2017},
  publisher={The MIT Press}
}

@article{lorch2021dibs,
  title={Dibs: Differentiable bayesian structure learning},
  author={Lorch, Lars and Rothfuss, Jonas and Sch{\"o}lkopf, Bernhard and Krause, Andreas},
  journal={Advances in Neural Information Processing Systems},
  volume={34},
  pages={24111--24123},
  year={2021}
}

@article{colombo2012learning,
  title={Learning high-dimensional directed acyclic graphs with latent and selection variables},
  author={Colombo, Diego and Maathuis, Marloes H and Kalisch, Markus and Richardson, Thomas S},
  journal={Annals of Stats.},
  year={2012},
}

@article{sachs2005causal,
  title={Causal protein-signaling networks derived from multiparameter single-cell data},
  author={Sachs, Karen and Perez, Omar and Pe'er, Dana and Lauffenburger, Douglas A and Nolan, Garry P},
  journal={Science},
  year={2005},
}

@article{buhlmann2014cam,
  title={CAM: Causal additive models, high-dimensional order search and penalized regression},
  author={B{\"u}hlmann, Peter and Peters, Jonas and Ernest, Jan},
  journal={Annals of Stats.},
  year={2014},
}

@inproceedings{ashman2022causal,
  title={Causal Reasoning in the Presence of Latent Confounders via Neural ADMG Learning},
  author={Ashman, Matthew and Ma, Chao and Hilmkil, Agrin and Jennings, Joel and Zhang, Cheng},
  booktitle={The Eleventh International Conference on Learning Representations},
  year={2022}
}

@article{ling2019bamb,
  title={BAMB: A balanced Markov blanket discovery approach to feature selection},
  author={Ling, Zhaolong and Yu, Kui and Wang, Hao and Liu, Lin and Ding, Wei and Wu, Xindong},
  journal={ACM Transactions on Intelligent Systems and Technology (TIST)},
  volume={10},
  number={5},
  pages={1--25},
  year={2019},
  publisher={ACM New York, NY, USA}
}

@article{meinshausen2006high,
  title={HIGH-DIMENSIONAL GRAPHS AND VARIABLE SELECTION WITH THE LASSO},
  author={MEINSHAUSEN, NICOLAI and B{\"U}HLMANN, PETER},
  journal={The Annals of Statistics},
  volume={34},
  number={3},
  pages={1436--1462},
  year={2006}
}

@article{eckstein1992douglas,
  title={On the Douglas—Rachford splitting method and the proximal point algorithm for maximal monotone operators},
  author={Eckstein, Jonathan and Bertsekas, Dimitri P},
  journal={Mathematical programming},
  volume={55},
  pages={293--318},
  year={1992},
  publisher={Springer}
}

@article{ma2013alternating,
  title={Alternating direction methods for latent variable Gaussian graphical model selection},
  author={Ma, Shiqian and Xue, Lingzhou and Zou, Hui},
  journal={Neural computation},
  volume={25},
  number={8},
  pages={2172--2198},
  year={2013},
  publisher={MIT Press}
}

@inproceedings{geffner2022deep,
  title={Deep End-to-end Causal Inference},
  author={Geffner, Tomas and Antoran, Javier and Foster, Adam and Gong, Wenbo and Ma, Chao and Kiciman, Emre and Sharma, Amit and Lamb, Angus and Kukla, Martin and Hilmkil, Agrin and others},
  booktitle={NeurIPS 2022 Workshop on Causality for Real-world Impact},
  year={2022}
}

@article{boyd2011distributed,
  title={Distributed optimization and statistical learning via the alternating direction method of multipliers},
  author={Boyd, Stephen and Parikh, Neal and Chu, Eric and Peleato, Borja and Eckstein, Jonathan and others},
  journal={Foundations and Trends{\textregistered} in Machine learning},
  volume={3},
  number={1},
  pages={1--122},
  year={2011},
  publisher={Now Publishers, Inc.}
}

@article{ng2020role,
  title={On the role of sparsity and dag constraints for learning linear dags},
  author={Ng, Ignavier and Ghassami, AmirEmad and Zhang, Kun},
  journal={Advances in Neural Information Processing Systems},
  volume={33},
  pages={17943--17954},
  year={2020}
}

@inproceedings{nazaret2024stable,
  title={Stable differentiable causal discovery},
  author={Nazaret, Achille and Hong, Justin and Azizi, Elham and Blei, David},
  booktitle={Proceedings of the 41st International Conference on Machine Learning},
  pages={37413--37445},
  year={2024}
}

\end{document}